\documentclass[11pt]{article}
\pdfoutput=1

\usepackage{amsmath,amssymb}
\usepackage{bm}
\usepackage{natbib}
\usepackage[usenames]{color}
\usepackage{amsthm}

\usepackage{multirow} 
\usepackage{enumerate}
\usepackage{bbm}

\usepackage[colorlinks,
linkcolor=red,
anchorcolor=blue,
citecolor=blue
]{hyperref}

\usepackage{setspace}
\usepackage[left=1in, right=1in, top=1in, bottom=1in]{geometry}

\usepackage{xcolor}

\theoremstyle{definition}
\newtheorem{theorem}{Theorem}[section]
\newtheorem{corollary}[theorem]{Corollary}
\newtheorem{lemma}[theorem]{Lemma}
\newtheorem{proposition}[theorem]{Proposition}

\newtheorem{claim}[theorem]{Claim}

\newtheorem{assumption}[theorem]{Assumption}
\theoremstyle{definition}

\newcommand{\N}{\mathbb{N}}

\newcommand{\R}{\mathbb{R}}

\newcommand{\E}{\mathbb{E}}

\let\P\BP

\let\hat\widehat
\newcommand{\f}{\frac}

\newcommand{\eps}{\varepsilon}
\renewcommand{\r}{\right}
\renewcommand{\l}{\left}

\newcommand{\pnorm}[2]{\left\|#1\right\|_{#2}}

\newcommand{\ip}[2]{\left\langle #1, #2 \right\rangle}

\newcommand{\ind}{{\mathbbm{1}}}

\newcommand{\summ}[2]{\sum_{#1 = 1}^{#2}}
\newcommand{\summm}[3]{\sum_{#1 = #2}^{#3}}

\newcommand{\proddd}[3]{\prod_{#1 = #2}^{#3}}

\newcommand{\supp}{\operatorname{supp}}

\newcommand{\bigOmega}[1]{\Omega\l(#1\r)}

\newcommand{\calD}{\mathcal{D}}
\newcommand{\calE}{\mathcal{E}}
\newcommand{\calF}{\mathcal{F}}

\newcommand{\calM}{\mathcal{M}}
\newcommand{\calN}{\mathcal{N}}

\newcommand{\calS}{\mathcal{S}}

\newcommand{\calV}{\mathcal{V}}
\newcommand{\calW}{\mathcal{W}}

\DeclareMathOperator{\tr}{tr}
\DeclareMathOperator{\Unif}{Unif}

\DeclareMathOperator{\poly}{poly}

\makeatletter
\renewcommand*{\eqref}[1]{%
  \hyperref[{#1}]{\textup{\tagform@{\ref*{#1}}}}%
}
\makeatother

\title{Algorithm-Dependent Generalization Bounds for Overparameterized Deep Residual Networks}
\author
{
    Spencer Frei\thanks{Department of Statistics, University of California, Los Angeles, CA 90095, USA; e-mail: {\tt spencerfrei@ucla.edu}}
    ~~~and~~~
	Yuan Cao\thanks{Department of Computer Science, University of California, Los Angeles, CA 90095, USA; e-mail: {\tt yuancao@cs.ucla.edu}} 
	~~~and~~~
	Quanquan Gu\thanks{Department of Computer Science, University of California, Los Angeles, CA 90095, USA; e-mail: {\tt qgu@cs.ucla.edu}}
}
\date{}

\begin{document}
\maketitle
\begin{abstract}

The skip-connections used in residual networks have become a standard architecture choice in deep learning due to the increased training stability and generalization performance with this architecture, although there has been limited theoretical understanding for this improvement. In this work, we analyze overparameterized deep residual networks trained by gradient descent following random initialization, and demonstrate that (i) the class of networks learned by gradient descent constitutes a small subset of the entire neural network function class, and (ii) this subclass of networks is sufficiently large to guarantee small training error.  By showing (i) we are able to demonstrate that deep residual networks trained with gradient descent have a small generalization gap between training and test error, and together with (ii) this guarantees that the test error will be small. Our optimization and generalization guarantees require overparameterization that is only logarithmic in the depth of the network, while all known generalization bounds for deep non-residual networks have overparameterization requirements that are at least polynomial in the depth.  This provides an explanation for why residual networks are preferable to non-residual ones.
\end{abstract}
\section{Introduction}
Deep learning has seen an incredible amount of success in a variety of settings over the past eight years, from image recognition~\citep{krizhevsky2012} to audio recognition~\citep{sainath2015} and more.  Compared with its rapid and widespread adoption, the theoretical understanding of why deep learning works so well has lagged significantly.  This is particularly the case in the common setup of an overparameterized network, where the number of parameters in the network greatly exceeds the number of training examples and input dimension.  In this setting, networks have the capacity to perfectly fit training data, regardless of if it is labeled with real labels or random ones~\citep{zhang2017}.  However, when trained on real data, these networks also have the capacity to truly learn patterns in the data, as evidenced by the impressive performance of overparameterized networks on a variety of benchmark datasets. This suggests the presence of certain mechanisms underlying the data, neural network architectures, and training algorithms which enable the generalization performance of neural networks.  
A theoretical analysis that seeks to explain why neural networks work so well would therefore benefit from careful attention to the specific properties that neural networks have when trained under common optimization techniques.

Many recent attempts at uncovering the generalization ability of deep learning focused on general properties of neural network function classes with fixed weights and training losses. For instance, \citet{bartlett2017} proved spectrally normalized margin bound for deep fully connected networks in terms of the spectral norms of the weights at each layer. \citet{neyshabur2018-snmb} proved a similar bound using PAC-Bayesian approach.  \citet{arora2018-compression} developed a compression-based framework for generalization of deep fully connected and convolutional networks, and also provided an explicit comparison of recent generalization bounds in the literature.  
All these studies involved algorithm-independent analyses of the neural network generalization, with resultant generalization bounds that involve quantities that make the bound looser with increased overparameterization.   

An important recent development in the practical deployment of neural networks has been the introduction of skip connections between layers, leading to a class of architectures known as residual networks.  Residual networks were first introduced by~\citet{he2015-resnet} to much fanfare, quickly becoming a standard architecture choice for state-of-the-art neural network classifiers.  The motivation for residual networks came from the poor behavior of very deep traditional fully connected networks: although deeper fully connected networks can clearly express any function that a shallower one can, in practice (i.e. using gradient descent) it can be difficult to choose hyperparameters that result in small training error.   Deep residual networks, on the other hand, are remarkably stable in practice, in the sense that they avoid getting stuck at initialization or having unpredictable oscillations in training and validation error, two common occurrences when training deep non-residual networks.  Moreover, deep residual networks have been shown to generalize with better performance and far fewer parameters than non-residual networks~\citep{tang2018,choi2019,squeezenet}.  We note that much of the recent neural network generalization literature has focused on non-residual architectures~\citep{bartlett2017, neyshabur2018-snmb, arora2018-compression, golowich2018, cao2019} with bounds for the generalization gap that grow exponentially as the depth of the network increases.  \citet{li2019} recently studied a class of residual networks and proved algorithm-independent bounds for the generalization gap that become larger as the depth of the network increases, with a dependence on the depth that is somewhere between sublinear and exponential (a precise characterization requires further assumptions and/or analysis).  We note that verifying the non-vacuousness of algorithm-independent generalization bounds relies on empirical arguments about what values the quantities that appear in the bounds generally take in practical networks (i.e. norms of weight matrices and interlayer activations), while algorithm-dependent generalization bounds such as the ones we provide in this paper can be understood without relying on experiments.

\subsection{Our Contributions}

In this work, we consider fully connected deep ReLU residual networks and study optimization and generalization properties of such networks that are trained with discrete time gradient descent following Gaussian initialization.  

 We consider binary classification under the cross-entropy loss and focus on data that come from distributions $\calD$ for which there exists a function $f$ for which $y\cdot f(x) \geq \gamma > 0$ for all $(x,y)\in \supp \calD$ from a large function class $\calF$ (see Assumption \ref{assumption:separability}).  By analyzing the trajectory of the parameters of the network during gradient descent, for any error threshold $\eps>0$, we are able to show:
\begin{enumerate}
    \item Under the cross-entropy loss, we can study an analogous surrogate error and bound the true classification error by the true surrogate error.  This method was introduced by~\citet{cao2019}.
    \item If $m^* = \tilde O(\poly(\gamma^{-1}))\cdot \max(d, \eps^{-2})$, then provided every layer of the network has at least $m\geq m^*$ units, gradient descent with small enough step size finds a point with empirical surrogate error at most $\eps$ in at most $\tilde O(\poly(\gamma^{-1})\cdot \eps^{-1})$ steps with high probability.  Here, $\tilde O(\cdot)$ hides logarithmic factors that may depend on the depth $L$ of the network, the margin $\gamma$, number of samples $n$, error threshold $\eps$, and probability level $\delta$.
    \item Provided $m^* = \tilde O(\poly(\gamma^{-1}, \eps^{-1}))$ and $n = \tilde O(\poly(\gamma^{-1}, \eps^{-1}))$, the difference between the empirical surrogate error and the true surrogate error is at most $\eps$ with high probability, and therefore the above provide a bound on the true classification error of the learned network.
\end{enumerate}
We emphasize that our guarantees above come with at most logarithmic dependence on the depth of the network.  Our methods are adapted from those used in the fully connected architecture by \citet{cao2019} to the residual network architecture.   The main proof idea is that overparameterization forces gradient descent-trained networks to stay in a small neighborhood of initialization where the learned networks (i) are guaranteed to find small surrogate training error, and (ii) come from a sufficiently small hypothesis class to guarantee a small generalization gap between the training and test errors.  By showing that these competing phenomena occur simultaneously, we are able to derive the test error guarantees of Corollary \ref{corollary:classification.error}.  The key insight of our analysis is that the Lipschitz constant of the network output for deep residual networks as well as the semismoothness property (Lemma \ref{lemma:semismoothness}) have at most logarithmic dependence on the depth, while the known analogues for non-residual architectures all have polynomial dependence on the depth.

\subsection{Additional Related Work}
In the last year there has been a variety of works developing algorithm-dependent guarantees for neural network optimization and generalization~\citep{liliang2018,allenzhu2018,zoucao2018,du2018-1layer,aroradu2019-finegrained,cao2019,zou2019improved,cao2019wide}. \citet{liliang2018} were among the first to theoretically analyze the properties of overparameterized fully connected neural networks trained with Gaussian random initialization, focusing on a two layer (one hidden layer) model under a data separability assumption.  Their work provided two significant insights into the training process of overparameterized ReLU neural networks: (1) the weights stay close to their initial values throughout the optimization trajectory, and (2) the ReLU activation patterns for a given example do not change much throughout the optimization trajectory.  These insights were the backbone of the authors' strong generalization result for stochastic gradient descent (SGD) in the two layer case.  The insights of~\citet{liliang2018} provided a basis to various subsequent studies.~\citet{du2018-1layer} analyzed a two layer model using a method based on the Gram matrix using inspiration from kernel methods, showing that gradient descent following Gaussian initialization finds zero training loss solutions at a linear rate.  \citet{zoucao2018} and \citet{allenzhu2018} extended the results of Li and Liang to the arbitrary $L$ hidden layer fully connected case, again considering (stochastic) gradient descent trained from random initialization.  Both authors showed that, provided the networks were sufficiently wide, arbitrarily deep networks would converge to a zero training loss solution at a linear rate, using an assumption about separability of the data. Recently, \citet{zou2019improved} provided an improved analysis of the global convergence of gradient descent and SGD for training deep neural networks, which enjoys a milder over-parameterization condition and better iteration complexity than previous work. Under the same data separability assumption,~\citet{zhang2019} showed that deep residual networks can achieve zero training loss for the squared loss at a linear rate with overparameterization essentially independent of the depth of the network. We note that~\citet{zhang2019} studied optimization for the regression problem rather than classification, and their results do not distinguish the case with random labels from that with true labels; hence, it is not immediately clear how to translate their analysis to a generalization bound for classification under the cross-entropy loss as we are able to do in this paper.

The above results provide a concrete answer to the question of why overparameterized deep neural networks can achieve zero training loss using gradient descent.  However, the theoretical tools of
\citet{du2018-1layer,allenzhu2018,zoucao2018,zou2019improved} apply to data with random labels as well as true labels, and thus do not explain the generalization to unseen data observed experimentally. ~\citet{dziugaiteroy2017} optimized PAC-Bayes bounds for the generalization error of a class of stochastic neural networks that are perturbations of standard neural networks trained by SGD.  
~\citet{cao2019} proved
a guarantee for arbitrarily small generalization error for classification in deep fully connected neural networks trained with gradient descent using random initialization.  The same authors recently provided an improved result for deep fully connected networks trained by stochastic gradient descent using a different approach that relied on the neural tangent kernel and online-to-batch conversion~\citep{cao2019wide}.  
\citet{weinan2019} recently developed algorithm-dependent generalization bounds for a special residual network architecture with many different kinds of skip connections by using kernel methods.  
  

\section{Network Architecture and Optimization Problem}
\label{sec:structure}
We begin with the notation of the paper.  We denote vectors by lowercase letters and matrices by uppercase letters, with the assumption that a vector $v$ is a column vector and its transpose $v^\top$ is a row vector.  We use the standard $O(\cdot), \Omega(\cdot), \Theta(\cdot)$ complexity notations to ignore universal constants, with $\tilde O(\cdot), \tilde \Omega(\cdot)$ additionally ignoring logarithmic factors.  For $n\in \N$, we write $[n]=\{1,2,\dots, n\}$. Denote the number of hidden units at layer $l$ as $m_l$, $l=1, \dots, L+1$.  Let the $l$-th layer weights be $W_l \in \R^{m_{l-1} \times m_{l}}$, and concatenate all of the layer weights into a vector $W = (W_1, \dots, W_{L+1})$.  Denote by $w_{l,j}$ the $j$-th column of $W_l$.  Let $\sigma(x) = \max(0,x)$ be the ReLU nonlinearity, and let $\theta$ be a constant scaling parameter. We consider a class of residual networks defined by the following architecture:
\begin{align*}
x_1 &= \sigma(W_1^\top x),\qquad x_l = x_{l-1} + \theta \sigma\left( W_l^\top x_{l-1}\right),\,\ l=2,\dots, L,\\
x_{L+1} &= \sigma(W_{L+1}^\top x_L).
\end{align*}
Above, we denote $x_l$ as the $l$-th hidden layer activations of input $x\in \R^d$, with $x_0 := x$.  In order for this network to be defined, it is necessary that $m_1=m_2=\dots=m_L$.  We are free to choose $m_{L+1}$, as long as $m_{L+1} = \Theta(m_1)$ (see Assumption \ref{assumption:width.same.order}).  We define a constant, non-trainable vector $v=(1, 1,\dots, 1, -1, -1, \dots, -1)^\top\in \R^{m_{L+1}}$ with equal parts $+1$ and $-1$'s that determines the network output,
\[ f_W(x) = v^\top x_{L+1}.\]
We note that our methods can be extended to the case of a trainable top layer weights $v$ by choosing the appropriate scale of initialization for $v$.  We choose to fix the top layer weights in this paper for simplicity of exposition.

We will find it useful to consider the matrix multiplication form of the ReLU activations, which we describe below.  Let $\ind(A)$ denote the indicator function of a set $A$, and define diagonal matrices $\Sigma_l(x) \in \R^{m_l \times m_l}$ by $[\Sigma_l(x)]_{j,j} = \ind( w_{l,j}^\top x_{l-1} > 0),\,\ l=1,\dots, L+1$. 
By convention we denote products of matrices $\proddd i a b M_i$ by $M_b \cdot M_{b-1} \cdot \ldots \cdot M_a$ when $a \leq b$, and by the identity matrix when $a>b$.  With this convention, we can introduce notation for the $l$-to-$l'$ interlayer activations $H_l^{l'}(x)$ of the network.  For $2 \leq l\leq l'\leq L$ and input $x\in \R^d$ we denote
\begin{align}
H_l^{l'}(x) &:= \proddd r l {l'} \l( I + \theta \Sigma_r(x) W_r^\top\r).&(2 \leq l\leq l' \leq L)
\label{eq:interlayer.activations.defn}
\end{align}
If $l =1 < l'$, we denote $H_1^{l'}(x) = H_2^{l'}(x) \Sigma_1(x) W_1^\top$, and if $l' = L+1 > l$, we denote $H_{l}^{L+1}(x) = \Sigma_{L+1}(x) W_{L+1}^\top H_l^L(x)$.
Using this notation, we can write the output of the neural network as $f_W(x) = v^\top H_{l+1}^{L+1}(x) x_l$ for any $l\in \{0\}\cup [L+1]$ and $x\in \R^d$.  For notational simplicity, we will denote $\Sigma_l(x)$ by $\Sigma_l$ and $H_l^{l'}(x)$ by $H_l^{l'}$ when the dependence on the input is clear.

We assume we have i.i.d. samples $(x_i, y_i)_{i=1}^n \sim \calD$ from a distribution $\calD$, where $x_i \in \R^d$ and $y_i \in \{\pm 1\}$.  We note the abuse of notation in the above, where $x_l \in \R^{m_l}$ refers to the $l$-th hidden layer activations of an arbitrary input $x\in \R^d$ while $x_i$ refers to the $i$-th sample $x_i\in \R^d$.  We shall use $x_{l,i}\in \R^{m_l}$ when referring to the $l$-th hidden layer activations of a sample $x_i\in \R^d$ (where $i\in [n]$ and $l\in [L+1]$), while $x_l\in \R^{m_l}$ shall refer to the $l$-th hidden layer activation of arbitrary input $x\in \R^d$. 

Let $\ell(x) = \log(1 + \exp(-x))$ be the cross-entropy loss.  We consider the empirical risk minimization problem optimized by constant step size gradient descent,
\begin{equation*}
    \min_{W} L_S(W):= \f 1 n \summ i n \ell(y_i \cdot f_W(x_i)),\qquad W_l^{(k+1)} = W_l^{(k)} - \eta \cdot \nabla_{W_l} L_S(W^{(k)})\quad (l\in [L+1]).
\end{equation*}
We shall see below that a key quantity for studying the trajectory of the weights in the above optimization regime is a surrogate loss defined by the derivative of the cross-entropy loss.  We denote the empirical and true surrogate loss by
\begin{equation*}
\calE_S(W) := - \f 1 n \summ i n \ell'(y_i \cdot f_W(x_i)),\quad \calE_D(W) := \E_{(x,y)\sim \calD} [-\ell'(y \cdot f_W(x))],
\end{equation*}
respectively.  The empirical surrogate loss was first introduced by~\citet{cao2019} for the study of deep non-residual networks.  
Finally, we note here a formula for the gradient of the output of the network with respect to different layer weights:
\begin{align}
\nabla_{W_l} f_W(x) &= \theta^{\ind(2 \leq l \leq L)} x_{l-1} v^\top H_{l+1}^{L+1} \Sigma_l(x), &(1 \leq l \leq L+1).
\label{eq:gradient.formulas}
\end{align}

\section{Main Theory}
\label{sec:main.theory}
We first go over the assumptions necessary for our proof and then shall discuss our main results.  Our assumptions align with those made by~\citet{cao2019} in the fully connected case.  The first main assumption is that the input data is normalized.
\begin{assumption}
Input data are normalized: $\supp(\calD_x) \subset S^{d-1} = \{x\in \R^d : \pnorm x2 =1\}$.
\label{assumption:normalized.input}
\end{assumption}
Data normalization is common in statistical learning theory literature, from linear models up to and including recent work in neural networks~\citep{liliang2018, zoucao2018, du2018-1layer, allenzhu2018, aroradu2019-finegrained, cao2019}, and can easily be satisfied for arbitrary training data by mapping samples $x\mapsto x/\pnorm{x}2$.

The next assumption is on the data generating distribution.  Because overparameterized networks can memorize data, any hope of demonstrating that neural networks have a small generalization gap must restrict the class of data distribution processes to one where some type of learning is possible. 
\begin{assumption}
Let $p(u)$ denote the density of a standard $d$-dimensional Gaussian vector.  Define
\[ \calF = \Bigg\{ \int_{\R^d} c(u) \sigma(u^\top x) p(u) \mathrm du :\ \pnorm{c(\cdot)}\infty \leq 1\Bigg\}.\]
Assume there exists $f(\cdot) \in \calF$ and constant $\gamma>0$ such that $y\cdot f(x) \geq \gamma$ for all $(x,y)\in \supp(\calD)$. 
\label{assumption:separability}
\end{assumption}
Assumption \ref{assumption:separability} was introduced by~\citet{cao2019} for the analysis of fully connected networks and is applicable for distributions where samples can be perfectly classified by the random kitchen sinks model of~\citet{rahimi2009}.  One can view a function from this class as the infinite width limit of a one-hidden-layer neural network with regularizer given by a function $c(\cdot)$ with bounded $\ell^\infty$-norm.  As pointed out by~\citet{cao2019}, this assumption includes the linearly separable case. 

Our next assumption concerns the scaling of the weights at initialization. 
\begin{assumption}[Gaussian initialization] We say that the weight matrices $W_l \in \R^{m_{l-1} \times m_{l}}$ are generated via Gaussian initialization if each of the entries of $W_l$ are generated independently from $N(0,2/m_l)$. 
\label{assumption:gaussian.init}
\end{assumption}
This assumption is common to much of the recent theoretical analyses of neural networks~\citep{liliang2018,zoucao2018,allenzhu2018,du2018-1layer,aroradu2019-finegrained,cao2019} and is known as the He initialization due to its usage in the first ResNet paper by \citet{he2015-resnet}.  This assumption guarantees that the spectral norms of the weights are controlled at initialization.

Our last assumption concerns the widths of the networks we consider and allows us to exclude pathological dependencies between the width and other parameters that define the architecture and optimization problem.
\begin{assumption}[Widths are of the same order]
We assume $m_{L+1} = \Theta(m_L)$.  We call $m = m_L\wedge m_{L+1}$ the width of the network.
\label{assumption:width.same.order}
\end{assumption}

Our first theorem shows that provided we have sufficient overparameterization and sufficiently small step size, the iterates $W^{(k)}$ of gradient descent stay within a small neighborhood of their initialization.  Additionally, the empirical surrogate error can be bounded by a term that decreases as we increase the width $m$ of the network.
\begin{theorem}
Suppose $W^{(0)}$ are generated via Gaussian initialization and that the residual scaling parameter satisfies $\theta = 1 / \Omega(L)$.  For $\tau>0$, denote a $\tau$-neighborhood of the weights $W^{(0)}=(W^{(0)}_1, \dots, W^{(0)}_{L+1})$ at initialization by
\[ \calW(W^{(0)}, \tau) := \Big \{ W = (W_1, \dots, W_{L+1}) : \pnorm{W_{l} - W^{(0)}_l}F \leq \tau \ \forall l\in [L+1]\Big \}.\]
There exist absolute constants $\nu,\nu',\nu'', C, C' >0$ such that for any $\delta >0$, provided $\tau \leq \nu \gamma^{12} \l(\log m \r)^{-\f 32}$, $\eta \leq \nu' (\tau m^{-\f 12} \wedge \gamma^4 m^{-1})$, and $K \eta \leq \nu'' \tau^2 \gamma^{4}  \l( \log (n/ \delta )\r)^{-\f 12}$, then if the width of the network is such that,
\begin{equation*}
m \geq C'\l( \tau^{-\f 4 3} d \log \f {m}{\tau \delta} \vee d \log \f{mL}{\delta} \vee \tau^{-\f 23} (\log m)^{-1} \log \f L \delta \vee \gamma^{-2} \l(d \log \f 1 \gamma \vee \log \f L \delta\r) \vee \log \f n \delta\r)
\end{equation*}
then with probability at least $1-\delta$, gradient descent starting at $W^{(0)}$ with step size $\eta$ generates $K$ iterates $W^{(1)}, \dots, W^{(K)}$ that satisfy:
\begin{enumerate}[(i)]
\item $W^{(k)} \in \calW(W^{(0)}, \tau)$ for all $k \in [K]$.
\item There exists $k\in \{0,\dots, K-1\}$ with $\calE_S(W^{(k)}) \leq C \cdot m^{-\f 12} \cdot \l( K \eta\r)^{-\f 12} \l( \log \f n \delta \r)^{\f 14} \cdot \gamma^{-2}$.
\end{enumerate}
\label{theorem:gen.wk.in.tau}
\end{theorem}
This theorem allows us to restrict our attention from the large class of all deep residual neural networks to the reduced complexity class of those with weights that satisfy $W\in \calW(W^{(0)}, \tau)$.  Our analysis provides a characterization of the radius of this reduced complexity class in terms of parameters that define the network architecture and optimization problem.   Additionally, this theorem allows us to translate the optimization problem over the empirical loss $L_S(W)$ into one for the empirical surrogate loss $\calE_S(W^{(k)})$, a quantity that is simply related to the classification error (its expectation is bounded by a constant multiple of the classification error under 0-1 loss; see Appendix \ref{appendix:theorem.gen.gap.proof}).

Our next theorem characterizes the Rademacher complexity of the class of residual networks with weights in a $\tau$-neighborhood of the initialization.  Additionally, it connects the test accuracy with the empirical surrogate loss and the Rademacher complexity. 
\begin{theorem}
Let $W^{(0)}$ denote the weights at Gaussian initialization and suppose the residual scaling parameter satisfies $\theta = 1 / \Omega(L)$.  Suppose $\tau \leq 1$.  Then there exist absolute constants $C_1, C_2, C_3 >0$ such that for any $\delta>0$, provided 
\[m\geq C_1 \l(\tau^{-\f 23} (\log m)^{-1} \log (L/\delta) \vee \tau^{-\f 4 3} d \log (m/(\tau \delta)) \vee d \log( mL/\delta)\r), \]
then with probability at least $1-\delta$, we have the following bound on the Rademacher complexity,
\[ \mathfrak{R}_n\l( \big\{ f_W : W\in \calW(W^{(0)}, \tau) \big \} \r) \leq C_2 \l( \tau^{\f 43} \sqrt{m \log m} + \f{\tau \sqrt{m}}{\sqrt n}\r),\]
so that for all $W\in \calW(W^{(0)}, \tau)$,
\begin{align}
\P_{(x,y)\sim \calD}\l( y \cdot f_W(x) < 0\r) \leq 2\calE_S(W) + C_2 \l( \tau^{\f 43} \sqrt{m \log m} + \f{\tau \sqrt{m}}{\sqrt n}\r)+ C_3 \sqrt{ \f{ \log (1/\delta)}n}.
\label{eq:thm:gen.gap.bound}
\end{align}
\label{theorem:gen.gap}
\end{theorem}
We shall see in Section \ref{sec:proof.sketch} that we are able to derive the above bound on the Rademacher complexity by using a semi-smoothness property of the neural network output and an upper bound on the gradient of the network output.  Standard arguments from statistical learning theory provide the first and third terms in \eqref{eq:thm:gen.gap.bound}.  

The missing ingredients needed to realize the result of Theorem \ref{theorem:gen.gap} for networks trained by gradient descent are supplied by Theorem \ref{theorem:gen.wk.in.tau}, which gives (i) control of the growth of the empirical surrogate error $\calE_S$ along the gradient descent trajectory, and (ii) the distance $\tau$ from initialization before which we are guaranteed to find small empirical surrogate error.  Putting these together yields Corollary \ref{corollary:classification.error}.
\begin{corollary}\label{corollary:classification.error}
Suppose that the residual scaling parameter satisfies $\theta = 1 / \Omega(L)$.  Let $\eps, \delta >0$ be fixed. Suppose that $m^* = \tilde O(\poly(\gamma^{-1})) \cdot \max(d, \eps^{-14}) \cdot \log (1/ \delta)$ and $n = \tilde O(\poly(\gamma^{-1})) \cdot \eps^{-4}$.  Then for any $m\geq m^*$, with probability at least $1-\delta$ over the initialization and training sample, there is an iterate $k\in \{0,\dots, K-1\}$ with $K = \tilde O(\poly(\gamma^{-1}))\cdot \eps^{-2}$ such that gradient descent with Gaussian initialization and step size $\eta = O(\gamma^4 \cdot m^{-1})$ satisfies
\[ \P_{(x,y)\sim D}[y \cdot f_{W^{(k)}} (x) < 0] \leq \eps.\]
\end{corollary}
This corollary shows that for deep residual networks, provided we have sufficient overparameterization, gradient descent is guaranteed to find networks that have arbitrarily high classification accuracy.  In comparison with the results of~\citet{cao2019}, the width $m$, number of samples $n$, step size $\eta$, and number of iterates $K$ required for the guarantees for residual networks given in Theorem \ref{theorem:gen.wk.in.tau} and Corollary \ref{corollary:classification.error} all have (at most) logarithmic dependence on $L$ as opposed to the exponential dependence in the corresponding results for the non-residual architecture.
Additionally, we note that the step size and number of iterations required for our guarantees are independent of the depth, and this is due to the advantage of the residual architecture. Our analysis shows that the presence of skip connections in the network architecture removes the complications relating to the depth that traditionally arise in the analysis of non-residual architectures for a variety of reasons.  The first is a technical one from the proof, in which we show that the Lipschitz constant of the network output and the semismoothness of the network depend at most logarithmically on the depth, so that the network width does not blow up as the depth increases (see Lemmas \ref{lemma:hidden.and.interlayer.activations.bounded} and \ref{lemma:semismoothness} below).  Second, the presence of skip-connections allows for representations that are learned in the first layer to be directly passed to later layers without needing to use a wider network to relearn those representations.  This property was key to our proof of the gradient lower bound of Lemma \ref{lemma:gradient.lower.bound} and has been used in previous approximation results for deep residual networks, e.g.,~\citet{yarotsky2017}.

\section{Proof Sketch of the Main Theory}
\label{sec:proof.sketch}
In this section we will provide a proof sketch of Theorems \ref{theorem:gen.wk.in.tau} and \ref{theorem:gen.gap} and Corollary \ref{corollary:classification.error}, following the proof technique of~\citet{cao2019}.  We will first collect the key lemmas needed for their proofs, leaving the proofs of these lemmas for Appendix \ref{appendix:key.lemma.proofs}.  We shall assume throughout this section that the residual scaling parameter satisfies $\theta = 1 / \Omega(L)$, which we note is a common assumption in the literature of residual network analysis~\citep{du2018-deep,allenzhu2018,zhang2019}.

Our first key lemma shows that the interlayer activations defined in \eqref{eq:interlayer.activations.defn} are uniformly bounded in $x$ and $l$ provided the network is sufficiently wide. 

\begin{lemma}[Hidden layer and interlayer activations are bounded] 
\label{lemma:hidden.and.interlayer.activations.bounded}
Suppose that $W_1, \dots, W_{L+1}$ are generated via Gaussian initialization.  Then there exist absolute constants $C_0,C_1,C_2> 0$ such that if $m\geq C_0 d \log \l( mL/\delta\r)$, then with probability at least $1-\delta$, for any $l,l'=1,\dots, L+1$ with $l\leq l'$ and $x\in S^{d-1}$, we have $C_1 \leq \pnorm{x_l}2 \leq C_2$ and $\pnorm{H_l^{l'}}2 \leq C_2$.
\end{lemma}
Due to the scaling of $\theta$, we are able to get bounds on the interlayer and hidden layer activations that do not grow with $L$.  As we shall see, this will be key for the sublinear dependence on $L$ for the results of Theorems  \ref{theorem:gen.wk.in.tau} and \ref{theorem:gen.gap}.  The fully connected architecture studied by \citet{cao2019} had additional polynomial terms in $L$ for both upper bounds for $\pnorm{x_l}2$ and $\pnorm{H_l^{l'}}2$.

Our next lemma describes a semi-smoothness property of the neural network output $f_W$ and the empirical loss $L_S$.  
\begin{lemma}[Semismoothness of network output and objective loss]
\label{lemma:semismoothness}
Let $W_1,\dots, W_{L+1}$ be generated via Gaussian initialization, and let $\tau \leq 1$.  Define 
\[ h(\hat W, \tilde W) := \pnorm{\hat W_1 - \tilde W_1}2 + \theta \sum_{l=2}^L \pnorm{\hat W_l - \tilde W_l}2 + \pnorm{\hat W_{L+1} - \tilde W_{L+1}}2.\]
There exist absolute constants $C, \overline{C}>0$ such that if 
$$m \geq C\l(\tau^{-\f 2 3} (\log m)^{-1} \log (L/\delta) \vee \tau^{-\f 4 3} d \log (m /(\tau \delta)) \vee d \log (mL/\delta)\r),$$ 
then with probability at least $1-\delta$, we have for all $x\in S^{d-1}$ and $\hat W, \tilde W \in \calW(W, \tau)$,
\begin{align*}
f_{\hat W} (x) - f_{\tilde W}(x) &\leq\overline{C} \tau^{\f 1 3} \sqrt{m \log m} \cdot  h(\hat W, \tilde W) + \overline{C} \sqrt{m} \cdot h(\hat W, \tilde W)^2\\
&\qquad+ \summ l {L+1} \tr \l[ \l( \hat W_l - \tilde W_l \r)^\top \nabla_{W_l} f_{\tilde W}(x)\r].
\end{align*}
and
\begin{align*}
L_S(\hat W) - L_S(\tilde W) &\leq \overline{C} \tau^{\f 13} \sqrt{m \log m}\cdot h(\hat W, \tilde W) \cdot \calE_S(\tilde W) + \overline C m \cdot h(\hat W, \tilde W)^2 \\
&\qquad + \sum_{l=1}^{L+1} \tr\l[ \l( \hat W_l - \tilde W_l\r)^\top  \nabla_{W_l} L_S(\tilde W) \r]. 
\end{align*}
\end{lemma}
The semismoothness of the neural network output function $f_W$ will be used in the analysis of generalization by Rademacher complexity arguments.  For the objective loss $L_S$, we apply this lemma for weights along the trajectory of gradient descent.  Since the difference in the weights of two consecutive steps of gradient descent satisfy $W_l^{(k+1)}-W_l^{(k)} = - \eta \nabla_{W_l} L_S(W^{(k)})$, the last term in the bound for the objective loss $L_S$ will take the form $-\eta  \summ l {L+1} \pnorm{\nabla_{W_l} L_S(W^{(k)})}F^2$.  Thus by simultaneously demonstrating (i) a lower bound for the gradient for at least one of the layers and (ii) an upper bound for the gradient at all layers (and hence an upper bound for $h(W^{(k+1)}, W^{(k)})$), we can connect the empirical surrogate loss $\calE_S(W^{(k)})$ at iteration $k$ with that of the objective loss $L_S(W^{(k)})$ that will lead us to Theorem \ref{theorem:gen.wk.in.tau}.    Compared with the fully connected architecture of \citet{cao2019}, our bounds do not have any polynomial terms in $L$.  

Thus the only remaining key items needed for our proof are upper bounds and lower bounds for the gradient of the objective loss, described in the following two lemmas.  

\begin{lemma}
\label{lemma:gradient.lower.bound}
Let $W = (W_1, \dots, W_{L+1})$ be weights at Gaussian initialization.   There exist absolute constants $C, \underline{C}, \nu$ such that for any $\delta >0$, provided $\tau \leq \nu \gamma^3$ and $m\geq C\gamma^{-2} \l( d \log \gamma^{-1} + \log (L /\delta)\r) \vee C \log (n/\delta)$, then with probability at least $1-\delta$, for all $\tilde W \in \calW(W, \tau)$, we have
\[ \pnorm{\nabla_{W_{L+1}} L_S(\tilde W)}F^2 \geq \underline{C} \cdot m_{L+1} \cdot \gamma^4 \cdot \calE_S(\tilde W)^2.\]
\end{lemma}
\begin{lemma}
\label{lemma:gradient.upper.bound}
Let $W = (W_1, \dots, W_{L+1})$ be weights at Gaussian initialization.  There exists an absolute constant $C>0$ such that for any $\delta>0$, provided $m\geq C \l( d \vee \log (L/\delta)\r)$ and $\tau \leq 1$, we have for all $\tilde W \in \calW(W,\tau)$ and all $l$,
\[ \pnorm{\nabla_{W_l} L_S(\tilde W)}F \leq \theta^{\ind(2 \leq l \leq L)}\cdot  C\sqrt{m} \cdot \calE_S(\tilde W).\]
\end{lemma}
Note that we provide only a lower bound for the gradient at the last layer.  It may be possible to improve the degrees of the polynomial terms of the results in Theorems \ref{theorem:gen.wk.in.tau} and \ref{theorem:gen.gap} by deriving lower bounds for the other layers as well.

With all of the key lemmas in place, we can proceed with a proof sketch of Theorems \ref{theorem:gen.wk.in.tau} and \ref{theorem:gen.gap}.  The complete proofs can be found in Appendix \ref{appendix:theorem.proofs}.

\begin{proof}[Proof of Theorem \ref{theorem:gen.wk.in.tau}]
Consider $h_k = h(W^{(k+1)}, W^{(k)})$, a quantity that measures the distance of the weights between gradient descent iterations.  It takes the form
\[ h_k = \eta \l[ \pnorm{\nabla_{W_1} L_S(W^{(k)})}2 + \theta \summm l 2 L \pnorm{\nabla_{W_l} L_S(W^{(k)})}2 + \pnorm{\nabla_{W_{L+1}} L_S(W^{(k)})}2 \r].\]
By Lemma \ref{lemma:gradient.upper.bound} we can show that $h_k \leq C \eta \sqrt{m} \calE_S(W^{(k)})$.  The gradient lower bound in Lemma \ref{lemma:gradient.lower.bound} substituted into Lemma \ref{lemma:semismoothness} shows that the dominating term in the semismoothness comes from the gradient lower bound, so that we have for any $k$,
\[ L_S(W^{(k+1)}) - L_S(W^{(k)}) \leq -C \cdot \eta \cdot m_{L+1} \cdot \gamma^4 \cdot \calE_S(W^{(k)})^2.\]
We can telescope the above over $k$ to get a bound on the loss at iteration $k$ in terms of the bound on the r.h.s. and the loss at initialization.  A simple concentration argument shows that the loss at initialization is small with mild overparameterization.  By letting $k^* = \mathrm{argmin}_{[K-1]} \calE_S(W^{(k)})^2$, we can thus show
\begin{equation*}
\calE_S(W^{(k^*)}) \leq C_3 \l( K \eta \cdot m \r)^{-\f 12} \l( L_S(W^{(0)})\r)^{\f 12} \cdot \gamma^{-2} \leq C_3 \l( K \eta \cdot m \r)^{-\f 12} \l( \log \f n \delta\r)^{\f 14} \cdot \gamma^{-2}.
\end{equation*}
\end{proof}

We provide below a proof sketch of the bound for the Rademacher complexity given in Theorem \ref{theorem:gen.gap}, leaving the rest for Appendix \ref{appendix:theorem.gen.gap.proof}.
\begin{proof}[Proof of Theorem \ref{theorem:gen.gap}]
Let $\xi_i$ be independent Rademacher random variables.  We consider a first-order approximation to the network output at initialization,
\[ F_{W^{(0)}, W}(x) := f_{W^{(0)}}(x) + \summ l {L+1} \tr \l[ \l(W_l - W_l^{(0)} \r)^\top \nabla_{W_l} f_{W^{(0)}} (x)\r],\]
and bound the Rademacher complexity by two terms,
\begin{align*}
\hat {\mathfrak{R}}_S [ \calF(W^{(0)}, \tau) ] &\leq \E_\xi \l[ \sup_{W \in  \calW(W^{(0)}, \tau)} \f 1 n \summ i n \xi_i [f(x_i)-F_{W^{(0)}, W}(x_i)] \r] \\
&\qquad+ \E_\xi \l[\sup_{W \in  \calW(W^{(0)}, \tau)} \f 1 n \summ i n \xi_i \summ l {L+1} \tr \l[ \l(W_l - W_l^{(0)} \r)^\top \nabla_{W_l} f_{W^{(0)}} (x)\r] \r]
\end{align*}
For the first term, taking $\tilde W = W^{(0)}$ in Lemma \ref{lemma:semismoothness} results in $|f_W(x) - F_{W^{(0)}, W}(x)| \leq C_3 \tau^{\f 4 3} \sqrt{m\log m}$.  For the second term, since $\pnorm{AB}F \leq \pnorm{A}F\pnorm{B}2$, we reduce this term to a product of two terms.  The first involves the norm of the distance of the weights from initialization, which is $\tau$.  The second is the norm of the gradient at initialization, which can be taken care of by using Cauchy--Schwarz and the gradient formula \eqref{eq:gradient.formulas} to get $\pnorm{\nabla_{W_l} f_{W^{(0)}}}F \leq C_2\theta^{\ind(2\leq \ell \leq L)} \sqrt{m}$.  A standard application of Jensen inequality gives the $1/\sqrt{n}$ term.
\end{proof}
Finally, we can put together Theorems \ref{theorem:gen.wk.in.tau} and \ref{theorem:gen.gap} by appropriately choosing the scale of $\tau$, $\eta$, and $K$ to get Corollary \ref{corollary:classification.error}.  We leave the detailed algebraic calculations for Appendix \ref{appendix:corollary.proof}.
\begin{proof}[Proof of Corollary \ref{corollary:classification.error}]
We need only specify conditions on $\tau,\eta, K\eta$, and $m$ such that the results of Theorems \ref{theorem:gen.wk.in.tau} and \ref{theorem:gen.gap} will hold, and making sure that each of the four terms in \eqref{eq:thm:gen.gap.bound} are of the same scale.  This can be satisfied by imposing the condition $K\eta = \nu '' \gamma^4 \tau^2 \l( \log(n/\delta)\r)^{-\f 12}$ and
\begin{align*}
C_3 \l( K \eta m\r)^{-\f 12} \l( \log (n/\delta)\r)^{\f 14} \cdot \gamma^{-2} = C_2 \tau^{\f 43} \sqrt{m \log m} = C_2 \tau \sqrt{m/n} = C_3 \sqrt{\log(1/\delta)/n} &=\eps/4.
\end{align*}
\end{proof}

\section{Conclusions}
\label{sec:conclusion}
In this paper, we derived algorithm-dependent optimization and generalization results for overparameterized deep residual networks trained with random initialization using gradient descent.  We showed that this class of networks is both small enough to ensure a small generalization gap and also large enough to achieve a small training loss.  Important to our analysis is the insight that the introduction of skip connections allows for us to essentially ignore the depth as a complicating factor in the analysis, in contrast with the well-known difficulty of achieving nonvacuous generalization bounds for deep non-residual networks.  This provides a theoretical understanding for the increased stability and generalization of deep residual networks over non-residual ones observed in practice.   

\section*{Acknowledgement}
This research was sponsored in part by the National Science Foundation IIS-1903202 and IIS-1906169. QG is also partially supported by the Salesforce Deep Learning Research Grant. The views and conclusions contained in this paper are those of the authors and should not be interpreted as representing any funding agencies.

\appendix
\section{Proofs of Main Theorems and Corollaries}
\label{appendix:theorem.proofs}
\subsection{Proof of Theorem \ref{theorem:gen.wk.in.tau}}
\label{appendix:theorem.wk.in.tau.proof}
We first show that $W^{(k)} \in \calW(W^{(0)}, \tau/2)$ for all $k\leq K$ satisfying $K\eta \leq \nu''\tau^2 \gamma^4(\log(n/\delta))^{-1/2}$.  Suppose $W^{(k')}\in \calW(W^{(0)}, \tau/2)$ for all $k'=1,\dots, k-1$.  By Lemma \ref{lemma:gradient.upper.bound}, we have
\[ \pnorm{\nabla_{W_{l}} L_S(W^{(k')})}F \leq C_1 \theta^{\ind(2 \leq l \leq L)} \sqrt{m}\cdot \calE_S(W^{(k')}).\]
Since $\eta \sqrt{m} \leq \nu' \tau$ and $\calE_S(\cdot)\leq 1$, we can make $\nu'$ small enough so that we have by the triangle inequality
\begin{equation}
\pnorm{W_l^{(k)} - W_l^{(0)}}F \leq \eta \pnorm{\nabla_{W_l} L_S(W^{(k-1)})}F + \f \tau 2 \leq \tau.
\label{eq:wlk-wl0-tau}
\end{equation}
Therefore we are in the $\tau$-neighborhood that allows us to apply the bounds described in the main section.  Define
\[ h_k := \eta \l[ \pnorm{\nabla_{W_1} L_S(W^{(k)})}2 + \theta \summm l 2 L \pnorm{\nabla_{W_l} L_S(W^{(k)})}2 + \pnorm{\nabla_{W_{L+1}} L_S(W^{(k)})}2 \r].\]
Then using the upper bounds for the gradient given in Lemma \ref{lemma:gradient.upper.bound}, we have
\begin{equation}
h_k \leq \eta \l[ C \sqrt{m} \calE_S(W^{(k)}) + \theta \summm l 2 L\l( \theta \sqrt{m} \calE_S(W^{(k)}) \r) + C \sqrt{m} \calE_S(W^{(k)})\r]\leq C' \eta \sqrt{m} \calE_S(W^{(k)}).
\label{eq:hk.upper.bound}
\end{equation}
Notice that $h_k=h(W^{(k+1)}, W^{(k)})$ where $h$ is from Lemma \ref{lemma:semismoothness}. Hence, we have
\begin{align*}
&L_S(W^{(k+1)}) - L_S(W^{(k)}) \\
&\leq C \tau^{\f 13} \sqrt{m \log m} \cdot h_k \cdot \calE_S(W^{(k)}) + C m h_k^2 - \eta \summ l {L+1} \pnorm{\nabla_{W_l} L_S(W^{(k)}) }F^2\\
&\leq C \eta \tau^{\f 13} \sqrt{m \log m} \cdot \sqrt{m} \cdot \calE_S(W^{(k)})^2 + C m^2 \eta^2 \cdot \calE_S(W^{(k)})^2 - C\eta \cdot m_{L+1} \cdot \gamma^4 \cdot \calE_S(W^{(k)})^2\\
&\leq \calE_S(W^{(k)})^2 \cdot \l( C_1 \eta \tau^{\f 1 3} m \sqrt{\log m} +  C_2 m^2 \cdot \eta^2  - C_3 \eta \cdot m_{L+1} \cdot \gamma^4 \r)
\end{align*}
The first inequality follows by Lemma \ref{lemma:semismoothness} and since $\tr(A^\top A) = \pnorm{A}F^2$.  The second inequality uses the lower bound for the gradient given in Lemma \ref{lemma:gradient.lower.bound} and \eqref{eq:hk.upper.bound}.  Therefore, if we take $\tau^{\f 13} \sqrt{\log m} \leq \nu^{\f 13} \gamma^4$, i.e. $\tau \leq \nu \cdot \gamma^{12} \l(\log m\r)^{-\f 32}$ for some small enough constant $\nu$, and if we take $\eta \leq \nu' \cdot \gamma^4 m^{-1}$, then there is a constant $C>0$ such that
\begin{equation}
L_S(W^{(k+1)}) - L_S(W^{(k)}) \leq -C \cdot \eta \cdot m_{L+1} \cdot \gamma^4 \cdot \calE_S(W^{(k)})^2.
\label{eq:empirical.surrogate.loss.small}
\end{equation}

Re-writing this we have
\begin{equation}
\calE_S(W^{(k)})^2 \leq C \gamma^{-4} \l( \eta m_{L+1}\r)^{-1} \l( L_S(W^{(k)}) - L_S(W^{(k+1)})\r).
\label{eq:empirical.surrogate.bounded.by.loss}
\end{equation}
Before completing this part of the proof, we will need the following bound on the loss at initialization: 
\begin{equation}
L_S(W^{(0)}) \leq C \sqrt{\log \f  n \delta}.
\label{eq:loss.at.init.bounded}
\end{equation}
To see this, we notice that $f_W(x_i)$ is a sum of $m_{L+1}/2$ independent random variables (conditional on $x_{L,i}$), 
\[ f_W(x_i) = \summ j {m_{L+1}/2} \l[ \sigma(w_{L+1,j}^\top x_{L,i}) - \sigma(w_{L+1,j+m_{L+1}/2}^\top x_{L,i})\r]. \]
Applying the upper bound for $\pnorm{x_{L+1}}2$ given by Lemma \ref{lemma:hidden.and.interlayer.activations.bounded} and Hoeffding inequality gives a constant $C_1>0$ such that with probability at least $1-\delta$, $|f_{W^{(0)}}(x_i)| \leq C_1 \sqrt{\log (n/\delta)}$ for all $i\in [n]$.  
Since $\ell(z) = \log(1 + \exp(-z)) \leq |z|+1$ for all $z\in \R$, we have
\[ L_S(W^{(0)}) = \f 1 n \summ i n \ell(y_i \cdot f_{W^{(0)}} (x_i)) \leq 1 + C_1 \sqrt{\log \f n \delta} \leq C \sqrt{\log(n/\delta)}.\]
We can thus bound the distance from initialization by
\begin{align*}
\pnorm{W_l^{(k)} - W_l^{(0)}}F &\leq \eta \summm {k'}0{k-1} \pnorm{\nabla_{W_l} L_S(W^{(k')})}F\\
&\leq C\eta \sqrt m \summm {k'}0 {k-1} \calE_S(W^{(k')})\\
&\leq C \eta \sqrt{m} \sqrt{k} \sqrt{ \gamma^{-4} \l( \eta m_{L+1}\r)^{-1} \summm {k'}0{k-1}\l( L_S(W^{(k)}) - L_S(W^{(k+1)})\r)}\\
&\leq C\sqrt{k\eta} \cdot \gamma^{-2} \l( \log \f n \delta\r)^{\f 14}\\
&\leq \f \tau 2.
\end{align*}
The first line comes from the definition of gradient descent and the triangle inequality.  For the second line, \eqref{eq:wlk-wl0-tau} allows us to apply Lemma \ref{lemma:gradient.upper.bound}.  The third line follows by Cauchy--Schwarz and \eqref{eq:empirical.surrogate.bounded.by.loss}.  The next line follows by \eqref{eq:loss.at.init.bounded}, and the last since $k\eta \leq \nu'' \tau^2 \gamma^4 (\log (n/\delta))^{-\f 12}$.  This completes the induction and shows that $W^{(k)} \in \calW(W^{(0)}, \tau)$ for all $k\leq K$.

For the second part of the proof, we want to derive an upper bound on the lowest empirical surrogate error over the trajectory of gradient descent.  Since we have shown that $W^{(k)} \in \calW(W^{(0)}, \tau/2)$ for $k\leq K$, \eqref{eq:empirical.surrogate.loss.small} and \eqref{eq:loss.at.init.bounded} both hold.  Let $k^* = \mathrm{argmin}_{k\in \{0,\dots, K-1\}} \calE_S(W^{(k)})^2$.  Then telescoping \eqref{eq:empirical.surrogate.loss.small} over $k$ yields
\begin{align}
\nonumber
L_S(W^{(K)}) - L_S(W^{(0)}) &\leq -C \cdot \eta \cdot m_{L+1} \cdot \gamma^4 \cdot \summ k K\calE_S(W^{(k)})^2\\
&\leq -C \cdot K\eta \cdot m_{L+1} \cdot \gamma^4 \cdot \calE_S(W^{(k^*)})^2.
\nonumber
\end{align}
Rearranging the above gives
\begin{equation*}
\calE_S(W^{(k^*)}) \leq C_3 \l( K \eta \cdot m \r)^{-\f 12} \l( L_S(W^{(0)})\r)^{\f 12} \cdot \gamma^{-2} \leq C_3 \l( K \eta \cdot m \r)^{-\f 12} \l( \log \f n \delta\r)^{\f 14} \cdot \gamma^{-2},
\end{equation*}
where we have used that $L_S(\cdot)$ is always nonnegative in the first inequality and \eqref{eq:loss.at.init.bounded} in the second.

\subsection{Proof of Theorem \ref{theorem:gen.gap}}
\label{appendix:theorem.gen.gap.proof}
Denote $\calF(W^{(0)}, \tau) = \{ f_W(x) : W\in \calW(W^{(0)}, \tau)\}$, and recall the definition of the empirical Rademacher complexity,
\begin{equation}
\hat {\mathfrak{R}}_S [\calF(W^{(0)}, \tau)]= \E_\xi \l[ \sup_{f\in \calF(W^{(0)}, \tau)} \f 1 n \summ i n \xi_i f(x_i) \r] =  \E_\xi \l[ \sup_{W \in  \calW(W^{(0)}, \tau)} \f 1 n \summ i n \xi_i f(x_i) \r],
\label{eq:gen.empirical.rademacher.defn}
\end{equation}
where $\xi = (\xi_1, \dots, \xi_n)^\top$ is an $n$-dimensional vector of i.i.d. $\xi_i \sim \Unif(\{ \pm 1 \})$.  Since $y\in \{\pm 1\}$, $|\ell'(z)|\leq 1$ and $\ell'(\cdot)$ is $1$-Lipschitz, standard uniform convergence arguments (see, e.g., \citet{shalevschwartz}) yield that with probability at least $1-\delta$,
\begin{align*}
\sup_{W\in \calW(W^{(0)}, \tau)} \l| \calE_S(W) - \calE_\calD(W)\r| &\leq 2 \E_{S} \hat {\mathfrak{R}}_S\l[\calF(W^{(0)}, \tau)\r] + C_1 \sqrt{ \f{\log(1/\delta)}n}.
\end{align*}
Since $-\ell'(x) = (1+\exp(-x))^{-1}$ satisfies $-\ell'(x) < \f 1 2$ if and only if $x < 0$, Markov's inequality gives
\begin{align*}
\P_{(x,y)\sim D} \l( y\cdot f_W(x) < 0\r) &\leq 2 \E_{(x,y)\sim \calD} \l( - \ell'(y\cdot f_W(x)) \r) = 2 \calE_D(W),
\end{align*}
so that it suffices to get a bound for the empirical Rademacher complexity \eqref{eq:gen.empirical.rademacher.defn}.  If we define
\[ F_{W^{(0)}, W}(x) := f_{W^{(0)}}(x) + \summ l {L+1} \tr \l[ \l(W_l - W_l^{(0)} \r)^\top \nabla_{W_l} f_{W^{(0)}} (x)\r],\]
then since $\sup_{a+b\in A+B} (a+b) \leq \sup_{a\in A} a + \sup_{b\in B} b$, we have
\begin{align*}
\hat{ \mathfrak{R}}_S [ \calF(W^{(0)}, \tau) ] &\leq \underbrace{\E_\xi \l[ \sup_{W \in  \calW(W^{(0)}, \tau)} \f 1 n \summ i n \xi_i [f(x_i)-F_{W^{(0)}, W}(x_i)] \r]}_{I_1} \\
&\qquad+ \underbrace{ \E_\xi \l[\sup_{W \in  \calW(W^{(0)}, \tau)} \f 1 n \summ i n \xi_i \summ l {L+1} \tr \l[ \l(W_l - W_l^{(0)} \r)^\top \nabla_{W_l} f_{W^{(0)}} (x)\r] \r]}_{I_2}\\
\end{align*}
For the $I_1$ term, we take $\tilde W = W^{(0)}$ in Lemma \ref{lemma:semismoothness} to get
\begin{align*}
|f_W(x) - F_{W^{(0)}, W}(x)| &\leq C \l[ \tau^{\f 4 3} \sqrt{m \log m} (2 + L \theta) \r] + C \tau^2 \sqrt{m} \l( 2 + L \theta \r) \\
&\leq C \tau^{\f 4 3} \sqrt{m\log m}.
\end{align*}

For $I_2$, since $\pnorm{AB}F \leq \pnorm{A}F\pnorm{B}2$, Lemma \ref{lemma:hidden.and.interlayer.activations.bounded} yields for all $l$ and any matrix $\xi$, 
\begin{equation}
\pnorm{x_l v^\top \cdot \xi}F \leq \pnorm{x_l v^\top}F \pnorm{\xi}2 \leq C \sqrt{m}\pnorm{\xi}2.
\nonumber
\end{equation}
Applying this to the gradient of $f$ at initialization given by \eqref{eq:gradient.formulas} and using Lemma \ref{lemma:hidden.and.interlayer.activations.bounded}, there is a constant $C_2$ such that
\begin{equation}
 \pnorm{\nabla_{W_l} f_{W^{(0)}}}F \leq  C_2 \theta^{\ind(2 \leq l \leq L)} \sqrt{m}.
\label{eq:gradient.norm.bound}
\end{equation}
We can therefore bound $I_2$ as follows:
\begin{align*}
I_2 &\leq \f \tau n \summ l {L+1} \E_\xi \pnorm{\summ i n \xi_i \nabla_{W_l} f_{W^{(0)}}(x_i) }F \\
&\leq \f \tau n \summ l {L+1} \sqrt{\E\pnorm{\summ i n \xi_i \nabla_{W_l} f_{W^{(0)}} (x_i)}F^2}\\
&= \f \tau n \summ l {L+1} \sqrt{\summ i n \pnorm{\nabla_{W_l} f_{W^{(0)}}(x_i)}F^2}\\
&\leq C \f \tau n \l( \sqrt{n m} + \summm l 2 L \sqrt{n m \theta^2} + \sqrt{nm}\r)\\
&\leq C \sqrt{\f m n} \tau.
\end{align*}
The first line above follows since $\tr(A^\top B) \leq \pnorm{A}F \pnorm{B}F$ and $W\in \calW(W^{(0)}, \tau)$.  The second comes from Jensen inequality, with the third since $\xi_i^2 =1$.   The fourth line comes from \eqref{eq:gradient.norm.bound}, with the final inequality by the scale of $\theta$.   This completes the proof.

\subsection{Proof of Corollary \ref{corollary:classification.error}}
\label{appendix:corollary.proof}
We need only specify conditions on $\tau,\eta, K\eta$, and $m$ such that the results of Theorems \ref{theorem:gen.wk.in.tau} and \ref{theorem:gen.gap} will hold, and such that each of the four terms in \eqref{eq:thm:gen.gap.bound} are of the same scale $\eps$.  To get the two theorems to hold, we need $\tau \leq \nu \gamma^{12} \l( \log m\r)^{-\f 32}$, $\eta \leq \nu' (\gamma^4 m^{-1} \wedge \tau m^{-\f 12})$, $K\eta \leq \nu'' \tau^2 \gamma^4 \l( \log (n/\delta)\r)^{-\f 12}$, and
\[ m \geq C\l( \gamma^{-2} d \log \f 1 \gamma \vee \gamma^{-2} \log \f L \delta \vee d \log \f L \delta \vee \tau^{-\f 43} d \log \f{L}{\tau \delta} \vee \tau^{-\f 23} (\log m)^{-1} \log \f L \delta \vee \log \f n \delta\r).\]
We now find the appropriate scaling by first setting the upper bound for the surrogate loss given in Theorem \ref{theorem:gen.wk.in.tau} to $\eps$ and then ensuring $\tau$ is such that the inequality required for $K\eta$ is satisfied:
\begin{align*}
C_3 \l( K \eta m\r)^{-\f 12} \l( \log (n/\delta)\r)^{\f 14} \cdot \gamma^{-2} &= \eps,\qquad K\eta = \nu '' \gamma^4 \tau^2 \l( \log(n/\delta)\r)^{-\f 12}.
\end{align*}
Substituting the values for $K\eta$ above, we get $C_4 m^{-\f 12} \gamma^{-2} \tau^{-1} \sqrt{\log(n/\delta)} = \eps$, so that 
\begin{equation}
\tau = C_6 \gamma^{-4} \eps^{-1} m^{-\f 12}\sqrt{\log (n/\delta)} .
\label{eq:gen.tau.scale}
\end{equation}
Let $\hat m$ be such that $\nu \gamma^{12} \l( \log m \r)^{-\f 32} = \tau$, so that $m (\log m)^{-3} = C\nu^{-2} \gamma^{-32}  \l( \log (n/\delta) \r) \eps^{-2}$.  It is clear that such a $\hat m$ can be written $\hat m = \tilde \Omega(\poly(\gamma^{-1}))\cdot \eps^{-2}$.  Finally we set
\[ m^* = \max\l( \hat m, d \log \f{mL}\delta, \tau^{-\f 43} \log \f{m}{\tau \delta}\r).\]
By \eqref{eq:gen.tau.scale} we can write $\tau^{-\f 43} \log(m/(\tau \delta)) = \gamma^{\f {16}3} \l( \log(n/\delta)\r)^{-\f 23} \eps^{\f 43} m^{\f 23} \log \l(  m^{3/2}\gamma^4 \eps (\log (n/\delta))^{-\f 12}/\delta\r)$.  Thus we can take
\[ m^* = \tilde \Omega(\poly(\gamma^{-1}) ) \cdot \max(d, \eps^{-2}) \cdot \log \f 1 \delta.\]
Using \eqref{eq:gen.tau.scale} we see that $K = C\gamma^{-4} \l( \log(n/\delta)\r)^{\f 12} \eps^{-2}$ and $\eta\leq \nu' \gamma^4 m^{-1}$ gives the desired forms of $K$ and $\eta$ as well as the first term of \eqref{eq:thm:gen.gap.bound}.  For the second term of \eqref{eq:thm:gen.gap.bound}, we again use \eqref{eq:gen.tau.scale} to get $\tau^{\f 43} \sqrt{m \log m} \leq C \gamma^{-\f{16}{3}} \l( \log (n/\delta)\r)^{\f 23} \eps^{-\f 43} m^{-\f 16} = R \eps^{-\f 43} m^{-\f 1 6}$ where $R = \tilde O(\poly(\gamma^{-1}))$.  Since $\eps^{-\f 43} m^{-\f 16} \leq \eps$ iff $m \geq \eps^{-14}$, this takes care of the second term in \eqref{eq:thm:gen.gap.bound}.  For the third term, we again use \eqref{eq:gen.tau.scale} to write $\tau \sqrt{m/n} = C \gamma^{-4} \sqrt{\log(n/\delta)}n^{-\f 12} \eps^{-1} \leq \eps$, which happens if $\sqrt{n/\log(n/\delta)}\geq C \eps^{-2} \gamma^{-4}$, i.e., $n = \tilde O(\mathrm{poly}(\gamma^{-1})) \eps^{-4}$.  For the final term of \eqref{eq:thm:gen.gap.bound}, it's clear that $\sqrt{\log(1/\delta) / n} \leq \eps$ is satisfied when $n \geq C \eps^{-2} \log(1/\delta)$, which is less stringent than the $\tilde O(\poly(\gamma^{-1}))\eps^{-4}$ requirement.

\section{Proofs of Key Lemmas}
\label{appendix:key.lemma.proofs}
In this section we provide proofs to the key lemmas discussed in Section \ref{sec:proof.sketch}.  We shall first provide the technical lemmas needed for their proof, and leave the proofs of the technical lemmas for Appendix \ref{appendix:technical.lemma.proofs}.  Throughout this section, we assume that $\theta = 1 / \Omega(L)$.

\subsection{Proof of Lemma \ref{lemma:hidden.and.interlayer.activations.bounded}: hidden and interlayer activations are bounded}
We first recall a standard result from random matrix theory; see, e.g. \citet{vershynin}, Corollary 5.35.
\begin{lemma}
Suppose $W_1,\dots, W_{L+1}$ are generated by Gaussian initialization.  Then there exist constants $C, C'>0$ such that for any $\delta > 0$, if $m \geq d\vee C \log (L/\delta)$, then with probability at least $1-\delta$, $\pnorm {W_l}2 \leq C'$ for all $l\in [L+1]$.  
\label{lemma:init.weight.norm}
\end{lemma}

The next lemma bounds the spectral norm of the maps that the residual layers define.  This is a key result that allows for the simplification of many of the arguments that are needed in non-residual architectures.  Its proof is in Appendix \ref{appendix:lemma.init.intermediate.proof}.

\begin{lemma}
\label{lemma:init.intermediate}
Suppose $W_1,\dots, W_L$ are generated by Gaussian initialization.  Then for any $\delta > 0$, there exist constants $C_0, C_0', C$ such that if $m \geq C_0 \log \l( L/\delta\r)$, then with probability at least $1-\delta$, for any $L\geq b \geq a \geq 2$, and for any tuple of diagonal matrices $\tilde \Sigma_a, \dots, \tilde \Sigma_b$ satisfying $\pnorm{\tilde \Sigma_i}2\leq 1$ for each $i=a,\dots, b$, we have
\begin{equation}
\pnorm{(I+\theta \tilde \Sigma_{b} W_b^\top) (I + \theta \tilde \Sigma_{b-1} W_{b-1}^\top)\cdot \ldots \cdot  (I + \theta \tilde \Sigma_a W_a^\top)}2 \leq \exp\l( C_0' \theta L \r) \leq 1.01.
\end{equation}
In particular, if we consider $\tilde \Sigma_i = \Sigma_i(x)$ for any $x\in S^{d-1}$, we have with probability at least $1-\delta$, for all $2\leq a \leq b \leq L$ and for all $x\in S^{d-1}$,
\[\pnorm{(I+\theta \Sigma_{b}(x) W_b^\top) (I + \theta  \Sigma_{b-1}(x) W_{b-1}^\top)\cdot \ldots \cdot  (I + \theta \Sigma_a(x) W_a^\top)}2 \leq \exp\l( C_0' \theta L \r)\leq 1.01.\]
\end{lemma}

The next lemma we show concerns a Lipschitz property of the map $x\mapsto x_l$. 
Compared with the fully connected case, our Lipschitz constant does not involve any terms growing with $L$, which allows for the width dependence of our result to be only logarithmic in $L$.  Its proof is in Appendix \ref{appendix:lemma:init.lipschitz.lth.layer}.
\begin{lemma}
Suppose $W_1,\dots, W_L$ are generated by Gaussian initialization.  There are constants $C,C'>0$ such that for any $\delta >0$, if $m\geq C d \log (mL/\delta)$, then with probability at least $1-\delta$, $\pnorm{x_l - x_l'}2\leq C' \pnorm{x-x'}2$ for all $x,x'\in S^{d-1}$ and $l\in [L+1]$.
\label{lemma:init.lipschitz.lth.layer}
\end{lemma}
With the above technical lemmas in place, we can proceed with the proof of Lemma \ref{lemma:hidden.and.interlayer.activations.bounded}.
\begin{proof}[Proof of Lemma \ref{lemma:hidden.and.interlayer.activations.bounded}]
We first show that a bound of the form $\underline{C} \leq \pnorm{\hat x_l}2 \leq \overline{C}$ holds for all $\hat x$ in an $\eps$-net of $S^{d-1}$ and then use the Lipschitz property from Lemma \ref{lemma:init.lipschitz.lth.layer} to lift this result to all of $S^{d-1}$.  

Let $\calN^*$ be a $\tau_0$-net of $S^{d-1}$.  By applying Lemma A.6 of \citet{cao2019} to the first layer of our network, there exists a constant $C_1$ such that with probability at least $1-\delta/3$, we can take $m = \bigOmega{d \log\l(m/(\tau_0 \delta)\r)}$ large enough so that
\begin{equation}
\nonumber
\pnorm{\hat x_1}2 \leq 1 + C_1 \sqrt{ \f{ d \log \l( m /(\tau_0 \delta) \r)}m} \leq 1.004.
\end{equation}

If $2\leq l \leq L$, by an application of Lemma \ref{lemma:init.intermediate}, by taking $m$ larger we have with probability at least $1-\delta/3$, for all $2 \leq l\leq L, \hat x\in \calN^*$,
\begin{align}
\nonumber
\pnorm {\hat x_l}2 &= \pnorm{(I + \theta \Sigma_l(\hat x) W_l^\top) \cdots (I + \theta \Sigma_2(\hat x) W_2^\top) \Sigma_1(\hat x) W_1^\top \hat x}2\\
\nonumber
&\leq \pnorm{(I + \theta \Sigma_l(\hat x) W_l^\top) \cdots (I + \theta \Sigma_2(\hat x) W_2^\top)}2\pnorm{\hat x_1}2\\
&\leq  1.01 \cdot \l( 1 + C_1 \sqrt{ \f{ d \log \l( m /(\tau_0\delta)\r)}m}\r)\leq  1.015.
\nonumber
\end{align}
For the last fully connected layer, we can use a proof similar to that of Lemma A.6 in \citet{cao2019} using the above upper bound on $\pnorm{\hat x_L}2$ to get that with probability at least $1-\delta$, for any $l \in [L+1]$ and $\hat x \in \calN^*$, 
\begin{equation}
\pnorm{\hat x_l}2 \leq 1.02.
\label{eq:lth.layer.ub.eps.net}
\end{equation}
For any $x\in S^{d-1}$, there exists $\hat x \in \calN^*$ such that $\pnorm{x - \hat x}2 \leq \tau_0$.  By Lemma \ref{lemma:init.lipschitz.lth.layer}, this means that with probability at least $1-\delta/2$, $\pnorm{x_l - \hat x_l}2\leq C_1 \tau_0$ for some $C_1>0$, and this holds over all $\hat x \in \calN^*$.  Let $\tau_0 = 1/m$, so that $d \log\l( mL/(\tau_0 \delta)\r) \leq 2d \log(mL/\delta)$.   Then \eqref{eq:lth.layer.ub.eps.net} yields that with probability at least $1-\delta$, for all $x\in S^{d-1}$ and all $l \in [L+1],$
\begin{align*}
\pnorm{x_l}2 &\leq \pnorm{\hat x_l}2+ \pnorm{x_l - \hat x_l}2 \leq 1.02+ C_1/m \leq 1.024.
\end{align*}
As for the lower bound, we again let $\calN^*$ be an arbitrary $\tau_0$-net of $S^{d-1}$.  For $l=1$, we use Lemma A.6 in \citet{cao2019} to get constants $C,C'$ such that provided $m \geq C d \log\l(m/(\tau_0\delta)\r)$, then we have with probability at least $1-\delta/3$, for all $\hat x\in \calN^*$,
\begin{equation}
\pnorm{\hat x_l}2 \geq 1 - C' \sqrt{ dm^{-1} \log\l(3m/(\tau_0\delta)\r)} \qquad (l=1, 2, \dots, L).
\label{eq:init.firstfc.activ.lb}
\end{equation}
To see that the above holds for layers $2 \leq l \leq L$, we note that it deterministically holds that $\hat x_{l,j} \geq \hat x_{1,j}$ for such $l$ and all $j$.  For the final layer, we follow a proof similar to Lemma A.6 of \citet{cao2019} with an application of \eqref{eq:lth.layer.ub.eps.net} to get that with probability at least $1-\delta/3$,
\begin{align*}
    \pnorm{\hat x_{L+1}}2^2 &\geq \pnorm{\hat x_L}2^2 - C_3 \sqrt{dm^{-1} \log\l(3 / (\tau_0 \delta)\r)}.
\end{align*}
Thus $m = \Omega(d \log(m/(\tau_0 \delta))$ implies there is a constant $C_4$ such that with probability at least $1-\delta$, for all $l\in [L+1]$ and $\hat x\in \calN^*$,
\begin{equation}
\pnorm{\hat x_l}2 \geq C_4 > 0.
\label{eq:lth.layer.lb.eps.net}
\end{equation}
By Lemma \ref{lemma:init.lipschitz.lth.layer}, we have with probability at least $1-\delta$, for all $x\in S^{d-1}$,
\begin{align*}
\pnorm{x_l}2 \geq \pnorm{\hat x_l}2 - \pnorm{x_l - \hat x_l}2 \geq C_4 - C_1 \tau_0.
\end{align*}
Thus by taking $\tau_0$ to be a sufficiently small universal constant, we get the desired lower bound.

We now demonstrate the upper bound for $\pnorm{H_l^{l'}}2$.  Since $H_{l}^{l'} = x_{l'}$ when $l =1$, we need only consider the case $l > 1$.  If $l' \leq L$, then $H_l^{l'}$ appears in the bound for Lemma \ref{lemma:init.intermediate} and so we are done.  For $l' = L+1$, by Lemmas \ref{lemma:init.weight.norm} and \ref{lemma:init.intermediate} we have
\begin{align*}
\pnorm{H_l^{L+1}}2 &= \pnorm{\Sigma_{L+1}(x)W_{L+1}^\top \proddd r {l}L \l( I + \theta \Sigma_r(x) W_r^\top\r)}2 \\
&\leq \pnorm{\Sigma_{L+1}(x)}2 \pnorm{W_{L+1}}2 \pnorm{ \proddd r {l}L \l( I + \theta \Sigma_r(x) W_r^\top\r)}2 \leq C.
\end{align*}
\end{proof}

\subsection{Proof of Lemma \ref{lemma:semismoothness}: semismoothness}
To prove the semismoothness result, we need two technical lemmas.  The first lemma concerns a Lipschitz-type property with respect to the weights, along with a characterization of the changing sparsity patterns of the rectifier activations at each layer.  The second lemma characterizes how the neural network output behaves if we know that one of the initial layers has a given sparsity pattern.  This allows us to develop the desired semi-smoothness even though ReLU is non-differentiable.  The proof for Lemmas \ref{lemma:tau.nhood.lipschitz.weights} and \ref{lemma:tau.nhood.sparse.rhs} can be found in Appendix \ref{appendix:lemma.tau.nhood.lipschitz.weights} and \ref{appendix:lemma.tau.nhood.sparse.rhs}, respectively. 

\begin{lemma}
Let $W = (W_1,\dots, W_{L+1})$ be generated by Gaussian initialization, and let $\hat W = (\hat W_1, \dots, \hat W_{L+1}), \tilde W = (\tilde W_1, \dots, \tilde W_{L+1})$ be weight matrices such that $\hat W, \tilde W \in \calW(W,\tau)$.  For $x\in S^{d-1}$, let $\Sigma_l(x), \hat \Sigma_l(x), \tilde \Sigma_l(x)$ and $x_l, \hat x_l, \tilde x_l$ be the binary matrices and hidden layer outputs of the $l$-th layers with parameters $W, \hat W, \tilde W$ respectively.   There exist absolute constants $C_1, C_2, C_3$ such that for any $\delta>0$, if $m\geq C_1  \tau^{-\f 43}\cdot d\log (m/(\tau \delta))\vee C_1 d \log (mL/\delta)$, then with probability at least $1-\delta$, for any $x\in S^{d-1}$ and any $l\in [L+1]$, we have
\[ \pnorm{\hat x_l - \tilde x_l}2 \leq \begin{cases}
C_2 \pnorm{\hat W_1 - \tilde W_1}2, & l=1,\\
C_2 \pnorm{\hat W_1 - \tilde W_1}2 + \theta C_2 \summm r 2 l \pnorm{\hat W_{r} - \tilde W_r}2, & 2 \leq l \leq L,\\
C_2 \pnorm{\hat W_1 - \tilde W_1}2 + \theta C_2 \summm r 2 L \pnorm{\hat W_{r} - \tilde W_r}2 + C_2 \pnorm{\hat W_{L+1} - \tilde W_{L+1}}2, & l = L+1.\end{cases}
 \]
and
\[ \pnorm{\hat \Sigma_l(x) - \tilde \Sigma_l(x)}0\leq C_3 m \tau^{\f 23}.\]
\label{lemma:tau.nhood.lipschitz.weights}
\end{lemma}

\begin{lemma}
\label{lemma:tau.nhood.sparse.rhs}
Let $W_1, \dots, W_{L+1}$ be generated by Gaussian initialization.  Let $\tilde W_l$ be such that $\pnorm{W_l - \tilde W_l}2\leq \tau$ for all $l$, and let $\tilde \Sigma_l(x)$ be the diagonal activation matrices corresponding to $\tilde W_l$, and $\tilde H_l^{l'}(x)$ the corresponding interlayer activations defined in \eqref{eq:interlayer.activations.defn}.  Suppose that $\pnorm{\tilde \Sigma_l(x) - \Sigma_l(x)}0 \leq s$ for all $x\in S^{d-1}$ and all $l$.  Define, for $l \geq 2$ and $a\in \R^{m_{l-1}}$,
\[ g_l(a,x) := v^\top \tilde H_l^{L+1}(x) a.\]
Then there exists a constant $C>0$ such that for any $\delta >0$, provided $m \geq C \tau^{-\f 23} (\log m)^{-1} \log (L /\delta)$, 
we have with probability at least $1-\delta$ and all $2 \leq l \leq {L+1}$, 
\[ \sup_{\pnorm{x}2=\pnorm{a}2=1,\ \pnorm{a}0\leq s} |g_l(a,x)| \leq C_1 \l[ \tau \sqrt{m} + \sqrt{s \log m}\r].\]
\end{lemma}

In comparison with the fully connected case of \citet{cao2019}, our bounds in Lemmas \ref{lemma:tau.nhood.lipschitz.weights} and \ref{lemma:tau.nhood.sparse.rhs} do not involve polynomial terms in $L$, and the residual scaling $\theta$ further scales the dependence of the hidden layer activations on the intermediate layers.  

With the above two technical lemmas, we can proceed with the proof of Lemma \ref{lemma:semismoothness}.
\begin{proof}[Proof of semismoothness, Lemma \ref{lemma:semismoothness}]
Recalling the notation of interlayer activations $H_l^{l'}$ from \eqref{eq:interlayer.activations.defn}, we have for any $l\in [L+1]$ $f_{\hat W}(x) = v^\top \hat H_{l+1}^{L+1} \hat x_l$, where we have denoted $H_l^{l'}(x) = H_l^{l'}$ for notational simplicity.  Similarly, in what follows we denote $\Sigma(x)$ by $\Sigma$ with the understanding that each diagonal matrix $\Sigma$ still depends on $x$.  We have the decomposition 
\[\hat H_2^{L+1} \hat \Sigma_1 \hat W_1 x = \l( \hat H_2^{L+1} - \tilde H_2^{L+1} \r) \hat \Sigma_1 \hat W_1^\top x + \tilde H_2^{L+1} \hat \Sigma_1 \hat W_1^\top x,\]
and for $2 \leq l \leq L$,
\[ \hat H_l^{L+1} - \tilde H_l^{L+1}= \l( \hat H_{l+1}^{L+1} - \tilde H_{l+1}^{L+1}\r) \l( I + \theta \hat \Sigma_l \hat W_l^\top\r) + \theta \tilde H_{l+1}^{L+1} \l( \hat \Sigma_l \hat W_l^\top - \tilde \Sigma_l \tilde W_l^\top \r).\]
Thus we can write
\begin{align*}
\hat H_1^{L+1}(x) - \tilde H_1^{L+1}(x) &= \l( \hat H_2^{L+1} - \tilde H_2^{L+1}\r) \hat \Sigma_1\hat W_1^\top x + \tilde H_2^{L+1} \l( \hat \Sigma_1 \hat W_1^\top - \tilde \Sigma_1 \tilde W_1^\top \r) x\\
&= \l( \hat \Sigma_{L+1} \hat W_{L+1}^\top - \tilde \Sigma_{L+1} \tilde W_{L+1}^\top \r) \hat x_L \\
&\qquad+ \theta \summm l 2 L \tilde H_{l+1}^{L+1} \l( \hat \Sigma_l \hat W_l^\top - \tilde \Sigma_l \tilde W_l^\top \r)\hat x_{l-1} +\tilde H_2^{L+1} \l( \hat \Sigma_1 \hat W_1 - \tilde \Sigma_1 \tilde W_1 \r) x.
\end{align*}
We thus want to bound the quantity
\begin{align}
\nonumber
f_{\hat W}(x) - f_{\tilde W}(x) &= v^\top \l( \hat \Sigma_{L+1} \hat W_{L+1}^\top - \tilde \Sigma_{L+1} \tilde W_{L+1}^\top \r) \hat x_L &(T_1)\\
\nonumber
&\qquad+ \theta v^\top \l[\summm l 2 L \tilde H_{l+1}^{L+1} \l( \hat \Sigma_l \hat W_l^\top - \tilde \Sigma_l \tilde W_l^\top \r)\hat x_{l-1}\r] &(T_2)\\
\label{eq:semismooth.T1T2T3}
&\qquad+v^\top \l[ \tilde H_{2}^{L+1} \l( \hat \Sigma_1 \hat W_1 - \tilde \Sigma_1 \tilde W_1 \r) x\r]. &(T_3)
\end{align}
We deal with the three terms separately.  The idea in each is the same. 

\underline{\textbf{First term, $T_1$.}} 
We write this as the sum of three terms $v^\top(I_1+I_2+I_3)$, where
\begin{align}
\nonumber
&\l( \hat \Sigma_{L+1} \hat W_{L+1}^\top - \tilde \Sigma_{L+1} \tilde W_{L+1}^\top \r) \hat x_L \\
\label{eq:semismooth.I1I2I3}
&= \underbrace{\l( \hat \Sigma_{L+1} - \tilde \Sigma_{L+1} \r) \hat W_{L+1}^\top \hat x_{L}}_{I_1} + \underbrace{\tilde \Sigma_{L+1}\l( \hat W_{L+1}^\top - \tilde W_{L+1}^\top \r) \l( \hat x_{L} - \tilde x_{L} \r)}_{I_2} + \underbrace{\tilde \Sigma_{L+1} \l( \hat W_{L+1}^\top - \tilde W_{L+1}^\top\r)\tilde x_{L}}_{I_3}.
\end{align}
By directly checking the signs of the diagonal matrices, we can see that for any $l=1,\dots, L+1$,
\begin{align}
\pnorm{\l( \hat \Sigma_l - \tilde \Sigma_l\r) \hat W_l^\top \hat x_{l-1} }2 &\leq C_1 \pnorm{\hat W_l - \tilde W_l}2 + C_1 \pnorm{ \hat x_{l-1} - \tilde x_{l-1}}2.
\label{eq:hatsigma.minus.tildesigma.bound}
\end{align}
We will use Lemma \ref{lemma:tau.nhood.lipschitz.weights} to get specific bounds for each $l$. 
Denote $|\Sigma|$ as the entrywise absolute values of a diagonal matrix $\Sigma$, so that $|\Sigma|\Sigma  = \Sigma$ provided the diagonal entries are all in $\{0,\pm 1\}$.  Then we can write
\begin{align}
\nonumber
|v^\top I_1| &= \pnorm{v^\top \l| \hat \Sigma_{L+1} - \tilde \Sigma_{L+1} \r| \l( \hat \Sigma_{L+1} - \tilde \Sigma_{L+1} \r)\hat W_{L+1}^\top \hat x_L}2\\
\nonumber
&\leq C_3 \tau^{\f 1 3}\sqrt{m} \pnorm{ \l( \hat \Sigma_{L+1} - \tilde \Sigma_{L+1} \r)\hat W_{L+1}^\top \hat x_L}2 \\
&\leq C_3 \tau^{\f 13} \sqrt{m} \cdot  \l( C_1 \pnorm{\hat W_{L+1} - \tilde W_{L+1}}2 + C_1 \pnorm{\hat x_{L} - \tilde x_L}2\r)
\label{eq:t1.i1.term}
\end{align}
The first inequality follows by first noting that for any vector $a$ with $|a_i|\leq 1$ it holds that $\pnorm{v^\top a}2 \leq \pnorm{a}0^{\f 12}$, and then applying Lemma \ref{lemma:tau.nhood.lipschitz.weights} to get $\pnorm{\hat \Sigma_{L+1} - \tilde \Sigma_{L+1}}0\leq s=O\l(m\tau^{\f 23}\r)$.  The last line is by \eqref{eq:hatsigma.minus.tildesigma.bound}.

The $I_2$ term in \eqref{eq:semismooth.I1I2I3} follows from a simple application of Cauchy--Schwarz:
\begin{align}
|v^\top I_2| &\leq \sqrt{m} \cdot C  \cdot \pnorm{\hat W_{L+1}- \tilde W_{L+1}}2 \pnorm{\hat x_L - \tilde x_L}2.
\label{eq:t1.i2.term}
\end{align}
Putting together \eqref{eq:t1.i1.term} and \eqref{eq:t1.i2.term} shows that we can bound $T_1$ in \eqref{eq:semismooth.T1T2T3} by
\begin{align}
\nonumber
T_1 &\leq C_3 \tau^{\f 13} \sqrt{m} \cdot  \l( C_1 \pnorm{\hat W_{L+1} - \tilde W_{L+1}}2 + C_1 \pnorm{\hat x_{L} - \tilde x_L}2\r) + \sqrt{m} \cdot C  \cdot \pnorm{\hat W_{L+1}- \tilde W_{L+1}}2 \pnorm{\hat x_L - \tilde x_L}2\\
\nonumber
&\qquad+ v^\top \tilde \Sigma_{L+1}\l( \hat W_{L+1} - \tilde W_{L+1}\r)^\top \tilde x_L\\
\nonumber
&\leq C_3 \tau^{\f 13} \sqrt m \l( C_1 \pnorm{\hat W_{L+1} - \tilde W_{L+1}}2 + C_1' \l[ \pnorm{\hat W_1 - \tilde W_1}2 + \theta \summm r 2 L \pnorm{\hat W_r - \tilde W_r}2 \r] \r)\\
\nonumber
&\qquad+ C \sqrt{m} \pnorm{\hat W_{L+1} - \tilde W_{L+1}}2 \l( \pnorm{\hat W_1 - \tilde W_1}2 + \theta \summm r 2 L \pnorm{\tilde W_r - \hat W_r}2 \r)\\
&\qquad+ v^\top \tilde \Sigma_{L+1}\l( \hat W_{L+1} - \tilde W_{L+1}\r)^\top \tilde x_L.
\label{eq:semismooth.T1.term}
\end{align}

\underline{\textbf{Second term, $T_2$.}}
We again use a decomposition like \eqref{eq:semismooth.I1I2I3}:
\begin{align}
\nonumber
&\tilde H_{l+1}^{L+1} \l( \hat \Sigma_l \hat W_l^\top - \tilde \Sigma_l \tilde W_l^\top \r) \hat x_{l-1}\\
\label{eq:semismooth.T2.I1I2I3}
&= \underbrace{\tilde H_{l+1}^{L+1} \l( \hat \Sigma_{l} - \tilde \Sigma_{l} \r) \hat W_{l}^\top \hat x_{l-1}}_{I_1} + \underbrace{ \tilde H_{l+1}^{L+1} \tilde \Sigma_{l}\l( \hat W_{l}^\top - \tilde W_{l}^\top \r) \l( \hat x_{l-1} - \tilde x_{l-1} \r)}_{I_2} +  \underbrace{\tilde H_{l+1}^{L+1}\tilde \Sigma_{l} \l( \hat W_{l}^\top - \tilde W_{l}^\top\r)\tilde x_{l-1}}_{I_3}.
\end{align}
For $I_1$, we note that Lemma \ref{lemma:tau.nhood.lipschitz.weights} gives sparsity level $s = O(m\tau^{\f 23})$ for $\hat \Sigma_l - \tilde \Sigma_l$.  We thus proceed similarly as for the term $T_1$ to get 
\begin{align}
\nonumber
|v^\top I_1| &\leq \pnorm{v^\top \tilde \Sigma_{L+1} \tilde W_{L+1}^\top \tilde H_{l+1}^{L} \l| \hat \Sigma_{l} - \tilde \Sigma_{l} \r| \l( \hat \Sigma_{l} - \tilde \Sigma_{l} \r) \hat W_{l}^\top \hat x_{l-1}}2\\
&\leq C \tau^{\f 1 3} \sqrt{m\log m}\cdot \l( C_ 1 \pnorm{\hat W_l - \tilde W_l}2 + C_2 \pnorm{\hat x_{l-1} - \tilde x_{l-1}}2\r).
\nonumber
\end{align}
The above follows since $s \log m \geq C \log ( L/ \delta)$ holds for $s = m\tau^{\f 23}$, and we can hence apply Lemma \ref{lemma:tau.nhood.sparse.rhs} and \eqref{eq:hatsigma.minus.tildesigma.bound}.  The bound for the $I_2$ term again follows by Cauchy--Schwarz,
\begin{align}
\nonumber
|v^\top I_2| &\leq \sqrt{m} \cdot C \cdot \pnorm{\hat W_l - \tilde W_l}2 \pnorm{\hat x_{l-1} - \tilde x_{l-1}}2.
\end{align}
Thus, for the term $T_2$ in \eqref{eq:semismooth.T1T2T3} we have
\begin{align}
\nonumber
T_2 &\leq \theta \summm l 2 L \l( C_6 \tau^{\f 1 3} \sqrt{m \log m} \pnorm{\hat W_l - \tilde W_l}2 + C \tau^{\f 13} \sqrt{m\log m} \pnorm{\hat W_1 - \tilde W_1}2 \r)\\
&\qquad +\theta^2 \summm l 2 L \l( \tau^{\f 13} \sqrt{m\log m} \summm r 2 l \pnorm{\tilde W_r - \hat W_r}2 \r)\nonumber\\
\nonumber
&\qquad+ \theta \summm r 2 L  \sqrt{m} C \pnorm{\hat W_l - \tilde W_l}2 \l( \pnorm{\hat W_1 - \tilde W_1}2 + \theta \summm r l 2 \pnorm{\hat W_r - \tilde W_r}2 \r)\\
&\qquad+ \theta \summm l 2 L v^\top \tilde H_{l+1}^{L+1}\tilde \Sigma_{l} \l( \hat W_{l}^\top - \tilde W_{l}^\top\r)\tilde x_{l-1}.
\label{eq:semismooth.T2.term}
\end{align}
\underline{\textbf{Third term, $T_3$.}}
For $T_3$, we work on the quantity
\begin{align*}
\tilde H_2^{L+1} \l( \hat \Sigma_1 \hat W_1^\top - \tilde \Sigma_1 \tilde W_1^\top\r) x &= \tilde H_2^{L+1} \l( \hat \Sigma_1 - \tilde \Sigma_1 \r) \hat W_1^\top x + \tilde H_2^{L+1} \tilde \Sigma_1 \l( \hat W_1 - \tilde W_1 \r)x.
\end{align*}
Thus, we again have by Lemma \ref{lemma:tau.nhood.sparse.rhs},
\begin{align}
\nonumber
T_3 &\leq \pnorm{v^\top \tilde H_2^{L+1} \l| \hat \Sigma_1 - \tilde \Sigma_1\r|}2\pnorm{\l( \hat \Sigma_1 - \tilde \Sigma_1\r) \hat W_1 x}2 + v^\top \tilde H_2^{L+1} \tilde \Sigma_1 \l( \hat W_1 - \tilde W_1 \r)x\\
&\leq \tau^{\f 13} \sqrt{m \log m} \pnorm{\hat W_1 - \tilde W_1}2 +  v^\top \tilde H_2^{L+1} \tilde \Sigma_1 \l( \hat W_1 - \tilde W_1 \r)x.
\label{eq:semismooth.T3.term}
\end{align}
Using the linearity of the trace operator and that $\tr(ABC) = \tr(CAB) = \tr(BCA)$ for any matrices $A,B,C$ for which those products are defined, we can use the gradient formula \eqref{eq:gradient.formulas} to calculate for any $l\in [L+1]$,
\begin{equation}
\theta^{\ind(2 \leq l \leq L)} v^\top \tilde H_l^{L+1} \tilde \Sigma_l \l( \hat W_l - \tilde W_l\r)^\top \tilde x_{l-1} = \tr\l[ \l( \hat W_l - \tilde W_l \r)^\top \nabla_{W_l} f_{\tilde W}(x) \r].
\label{eq:trace.term.of.semismoothness}
\end{equation}
Let now
\[ h(\hat W, \tilde W) := \pnorm{\hat W_1 - \tilde W_1}2 + \theta \summm l 2 L \pnorm{\hat W_l - \tilde W_l}2 + \pnorm{\hat W_{L+1} - \tilde W_{L+1}}2.\]
Substituting the bounds from \eqref{eq:semismooth.T1.term}, \eqref{eq:semismooth.T2.term}, \eqref{eq:semismooth.T3.term} and \eqref{eq:trace.term.of.semismoothness} thus yield for some constant $\overline{C}$, 
\begin{align}
\nonumber
&f_{\hat W}(x) - f_{\tilde W}(x) \leq C \tau^{\f 13} \sqrt{m \log m} \l[ \pnorm{\hat W_1 - \tilde W_1}2 + \theta C \summm l 2 L \pnorm{\hat W_l - \tilde W_l}2 + C\pnorm{\hat W_{L+1} - \tilde W_{L+1}}2 \r]\\
\nonumber
&\qquad+ C\tau^{\f 13} \sqrt{m \log m}\l[ \pnorm{\hat W_1 - \tilde W_1}2 + C\pnorm{\hat W_{1} - \tilde W_{1}}2 + \theta C \summm l 2 l \pnorm{\hat W_l - \tilde W_l}2 \r]\\
\nonumber
&\qquad+ C \sqrt{m} \Bigg [ \pnorm{\hat W_{L+1} - \tilde W_{L+1}}2 \cdot \pnorm{\hat W_1 - \tilde W_{L+1}}2 + \theta \pnorm{\hat W_{L+1} - \tilde W_{L+1}}2 \summm r 2 L \pnorm{\hat W_r - \tilde W_r}2 \\
\nonumber
&\qquad+ \theta \summm l 2 L \pnorm{\hat W_l - \tilde W_l} 2\pnorm{\hat W_1 - \tilde W_1}2 + \theta \summm l 2 L \pnorm{\hat W_l - \tilde W_l}2 \cdot \l( \theta \summm r 2 l \pnorm{\hat W_r - \tilde W_r}2 \r) \Bigg]\\
\nonumber
&\qquad +\summ l {L+1} \tr \l[ \l( \hat W_l - \tilde W_l \r) \nabla_{W_l} f_{\tilde W}(x)\r]\\
&\leq \overline{C} \tau^{\f 1 3} \sqrt{m \log m} \cdot h(\hat W, \tilde W) +  \overline{C} \sqrt{m}\cdot h(\hat W, \tilde W)^2 + \summ l {L+1} \tr \l[ \l( \hat W_l - \tilde W_l \r) \nabla_{W_l} f_{\tilde W}(x)\r] \label{eq:fw.semismoothness}
\end{align}

This completes the proof of semi-smoothness of $f_W$.  For $L_S$, denote $\hat y_i,\tilde y_i$ as the outputs of the network for input $x_i$ under weights $\hat W, \tilde W$ respectively.  Since $\ell''(z) \leq 0.5$ for all $z\in \R$, if we denote $\Delta_i = \hat y_i - \tilde y_i = f_{\hat W}(x_i) - f_{\tilde W}(x_i)$, we have
\begin{align*}
L_S(\hat W) - L_S(\tilde W) \leq \f 1 n \summ in \l[ \ell'(y_i \tilde y_i) \cdot y_i \cdot \Delta_i + \f 14 \Delta_i^2\r].
\end{align*}
Applying \eqref{eq:fw.semismoothness} and using that $-n^{-1} \summ i n \ell'(z_i) \leq 1$ for any $z_i\in \R$,
\begin{align*}
 \f 1 n \summ i n \ell'(y_i \tilde y_i) y_i \cdot \Delta_i &\leq C \tau^{\f 13} \sqrt{m \log m}\cdot h(\hat W, \tilde W) \cdot \calE_S(\tilde W) + C \sqrt{m} \cdot h(\hat W, \tilde W)^2 \cdot \calE_S(\tilde W)\\
& \qquad+ \summ l {L+1} \f 1 n\summ i n  \ell'(y_i \tilde y_i) \cdot y_i \cdot \tr\l[ \l( \hat W_l - \tilde W_l\r) \nabla_{W_l} f_{\tilde W}(x_i) \r].
\end{align*}
Linearity of the trace operator allows the last term in the above display to be written as 
$$\summ l {L+1} \tr\l[ \l( \hat W_l - \tilde W_l\r) \nabla_{W_l} L_S(\tilde W) \r].$$   
Moreover, using Lemma \ref{lemma:tau.nhood.lipschitz.weights},
\begin{align*}
\Delta_i^2 &= \l[ v^\top (\hat x_{L+1,i} - \tilde x_{L+1,i})\r]^2 \leq \pnorm{v}2^2 \pnorm{\hat x_{L+1,i}-\tilde x_{L+1,i}}2^2 \leq C_2\cdot m \cdot h(\hat W, \tilde W)^2.
\end{align*}
This term dominates the corresponding $h^2$ term coming from $\Delta_i$ and so completes the proof.

\end{proof}

\subsection{Proof of Lemma \ref{lemma:gradient.lower.bound}: gradient lower bound}
This is the part of the proof that makes use of the assumption on the data distribution given in Assumption \ref{assumption:separability}, and is key to the mild overparameterization required for our generalization result.  The key technical lemma needed for the proof of the gradient lower bound is given below.  The proof of Lemma \ref{lemma:gen.lastlayer.lowerbound} can be found in Appendix \ref{appendix:lemma.gen.lastlayer.lowerbound}.
\begin{lemma}
\label{lemma:gen.lastlayer.lowerbound}
Let $a(x,y):S^{d-1} \times \{\pm 1\} \to [0,1]$.  For any $\delta >0$, there is a constant $C>0$ such that if $m\geq C\gamma^{-2} \l( d \log ( 1 /\gamma) + \log (L /\delta)\r)$ and $m\geq C \log (n/\delta)$ then for any such function $a$, we have with probability at least $1-\delta$,
\begin{align*}
\summ j {m_{L+1}} \pnorm{\f 1 n  \summ i n\l[ a(x_i, y_i) \cdot y_i \cdot \sigma'\l( w_{L+1,j}^\top x_{L,i}\r) \cdot x_{L,i} \r]}2^2 &\geq \f 1 {67} m_{L+1} \gamma^2 \l( \f 1 n \summ i n a(x_i,y_i)\r)^2.
\end{align*}
\end{lemma}
\begin{proof}[Proof of Lemma \ref{lemma:gradient.lower.bound}]
Let $\tilde y_i := f_{\tilde W}(x_i)$, and define $g_j := \f 1 n \summ i n \l[ \ell'(y_i \tilde y_i) \cdot v_j \cdot y_i \cdot \sigma'(w_{L+1,j}^\top x_{L,i}) \cdot x_{L,i}\r]$ so that
\[ \summ j {m_{L+1}} \pnorm{g_j}2^2 = \summ j {m_{L+1}} \pnorm{\f 1 n \summ i n \l[  \ell'(y_i \tilde y_i) \cdot y_i \cdot \sigma'(w_{L+1,j}^\top x_{L,i}) \cdot x_{L,i}\r] }2^2.\]
Recall that $\calE_S(\tilde W) = -n^{-1}\summ i n \ell'(y_i\tilde y_i)$.  Applying Lemma \ref{lemma:gen.lastlayer.lowerbound} gives
\begin{equation}
\summ j {m_{L+1}} \pnorm{g_j}2^2 \geq \f 1 {67} m_{L+1} \gamma^2[\calE_S(\tilde W)]^2.
\label{eq:gen.sumgjnorm.lb}
\end{equation}
By Lemma \ref{lemma:hidden.and.interlayer.activations.bounded}, for any $j\in [m_{L+1}]$, we have
\begin{equation}
\pnorm{g_j}2 \leq \f 1 n \summ i n \pnorm{\ell'(y_i \tilde y_i) \cdot v_j \cdot y_i \cdot \sigma'(w_{L+1,j}^\top x_{L,i}) \cdot x_{L,i}}2 \leq  1.02 \calE_S(\tilde W).
\label{eq:gen.gj.lb.uniform}
\end{equation}
Define
\[ A := \Big\{j\in [m_{L+1}] : \pnorm{g_j}2^2 \geq \f 1 {2\cdot 67} \gamma^2 \l(\calE_S(\tilde W))\r)^2\Big\}.\]
We can get the following lower bound on $|A|$:
\begin{align*}
|A| \calE_S(\tilde W)^2 &\geq \f 1 {1.02^2} \sum_{j\in A} \pnorm{g_j}2^2\\
&\geq \f 1 {1.05} \l( \f 1 {67} m_{L+1} \gamma^2 [\calE_S(\tilde W)]^2- \f 1 {2\cdot 67} |A^c|\gamma^2[\calE_S(\tilde W)]^2\r) \\
&\geq \f 1 {1.05\cdot 2 \cdot 67} m_{L+1} \gamma^2[\calE_S(\tilde W)]^2.
\end{align*}
The first line follows by \eqref{eq:gen.gj.lb.uniform}, and the second by writing the sum over $[m_{L+1}]$ as a sum over $A$ and $A^c$ and then \eqref{eq:gen.sumgjnorm.lb} and the definition of $A$.  The last line holds since $|A^c|\leq m_{L+1}$, and all of the above allows for the bound
\begin{equation}
|A| \geq \f 1 {141} m_{L+1} \gamma^2.
\label{eq:gen.A.lb}
\end{equation}
Let now $A' = \{ j \in [m_{L+1}] : \sigma'(\tilde w_{L+1,j}^\top \tilde x_{L,i}) \neq \sigma'(w_{L+1,j}^\top x_{L,i}) \}$.  By Lemma \ref{lemma:tau.nhood.lipschitz.weights}, we have
\begin{equation}
|A'| = \pnorm{\tilde \Sigma_{L+1}(x) - \Sigma_{L+1}(x) }0 \leq C_1\tau^{\f 23}m_{L+1}.
\label{eq:gen.A'.ub}
\end{equation}
Since $\tau \leq \nu \gamma^3$, we can make $\nu$ small enough so that $C_1 \tau^{\f 23} < \gamma^2\cdot \l(1/141 - 1/150\r)$.  Thus \eqref{eq:gen.A.lb} and \eqref{eq:gen.A'.ub} imply
\begin{align}
|A\setminus A'| &\geq |A| - |A'| \geq \f 1 {141} m_{L+1}\gamma^2 - C_1 \tau^{\f 23} m_{L+1} \geq \f 1 {150} m_{L+1} \gamma^2.
\label{eq:a.minus.a'.lb}
\end{align}
By definition, $\nabla_{W_{L+1,j}} L_S (\tilde W) = \f 1 n \summ i n \ell'(y_i \tilde y_i)\cdot v_j \cdot y_i \cdot \sigma'(\tilde w_{L+1,j}^\top \tilde x_{L,i} ) \cdot \tilde x_{L,i}$.  
For indices $j\in A\setminus A'$, we can therefore write
\begin{align}
\nonumber
\pnorm{g_j}2 - \pnorm{\nabla_{W_{L+1,j}} L_S(\tilde W)}2 &\leq  \pnorm{ \f 1 n \summ i n \ell '(y_i \tilde y_i) \cdot v_j \cdot y_i \cdot \sigma'(w_{L+1,j}^\top x_{L,i}) \cdot (x_{L,i} - \tilde x_{L,i})} 2\\
\nonumber
&\leq \f 1 n \summ i n \pnorm{\ell '(y_i \tilde y_i) \cdot v_j \cdot y_i \cdot \sigma'(w_{L+1,j}^\top x_{L,i}) \cdot (x_{L,i} - \tilde x_{L,i})} 2\\
\label{eq:gj.lsgradient.ub}
&\leq C_3 \tau \calE_S(\tilde W).
\end{align}
The first inequality follows by the triangle inequality and since indices $j\in A\setminus A'$ satisfy $\sigma'(\tilde w_{L+1,j}^\top \tilde x_{L,i})=\sigma(w_{L+1,j}^\top x_{L,i})$.   The second inequality is an application of Jensen inequality.  The last inequality follows by Lemma \ref{lemma:tau.nhood.lipschitz.weights} and since $v_j, y_i \in \{ \pm 1 \}$.  
Now take $\nu$ small enough so that $C_3 \tau < \l( (2\cdot 67)^{-1/2} - 1/ {16}\r)$.  Then we can use \eqref{eq:gj.lsgradient.ub} together with the definition of $A$ to get for any index $j\in A\setminus A'$,
\begin{align}
\pnorm{\nabla_{W_{L+1,j}} L_S(\tilde W)}2 &\geq  \f 1 {\sqrt{2\cdot 67}}\gamma\calE_S(\tilde W) - C_3 \tau \calE_S(\tilde W) \geq \f 1 {16} \gamma \calE_S(\tilde W).
\label{eq:gradient.last.layer.jth.column.lower.bound}
\end{align}
Thus we can derive the lower bound for the gradient of the loss at the last layer:
\begin{align*}
\pnorm{\nabla_{W_{L+1}} L_S (\tilde W)}F^2 &= \summ j {m_{L+1}}\pnorm{\nabla_{W_{L+1,j}} L_S (\tilde W)}F^2 \\
&\geq  \sum_{j\in A\setminus A'} \pnorm{ \nabla_{W_{L+1,j}} L_S(\tilde W)}2^2\\
&\geq  \f 1 {16^2}|A\setminus A'|  \gamma^2 [\calE_S(\tilde W)]^2\\
&\geq \f 1 {150\cdot 16^2} \gamma^4 m_{L+1} [\calE_S(\tilde W)]^2.
\end{align*}
The first line is by definition, and the second is since the spectral norm is at most the Frobenius norm.  The third line uses \eqref{eq:gradient.last.layer.jth.column.lower.bound}, and the final inequality comes from \eqref{eq:a.minus.a'.lb}.
\end{proof}

\subsection{Proof of Lemma \ref{lemma:gradient.upper.bound}: gradient upper bound}
\begin{proof}
Using the gradient formula \eqref{eq:gradient.formulas} and the $H_l^{l'}$ notation from \eqref{eq:interlayer.activations.defn}, we can write
\begin{equation}
\nabla_{W_l} L_S(\tilde W) = \theta^{\ind(2 \leq l \leq L)} \f 1 n \summ i n \ell'(y_i \tilde y_i) \cdot y_i \cdot \tilde x_{l-1,i} v^\top \tilde H_{l+1}^{L+1} \tilde \Sigma_{l}(x_i), \quad (1 \leq l \leq L+1).
\label{eq:gradient.loss.at.tilde.w}
\end{equation}
Since $\tau \leq 1$, there is a constant $C$ such that w.h.p. $\pnorm{\tilde W_l}2 \leq C$ for all $l$.  Thus, it is easy to see that an analogous version of Lemma \ref{lemma:init.intermediate} can be applied with Lemma \ref{lemma:tau.nhood.lipschitz.weights} to get that with probability at least $1-\delta$, for all $i\in [n]$ and for all $l$, 
\begin{equation}
\pnorm{\tilde x_{l-1,i}}2 \leq C_1\qquad \text{and} \qquad \pnorm{\tilde H_{l+1}^{L+1}}2 \leq C_2.
\label{eq:tilde.interlayer.upper.bounds}
\end{equation}
Therefore, we can bound
\begin{align*}
\pnorm{\nabla_{W_l} L_S(\tilde W)}F &\leq \f 1 n\summ i n \pnorm{\ell'(y_i\tilde y_i) \cdot y_i \cdot \tilde x_{l-1,i} v^\top \tilde H_{l+1}^{L+1} \tilde \Sigma_{l+1} (x_i)}F\\
&= \f 1 n\summ i n \pnorm{\ell'(y_i\tilde y_i) \cdot y_i \cdot \tilde x_{l-1,i} }2\pnorm{v^\top \tilde H_{l+1}^{L+1} \tilde \Sigma_{l+1} (x_i)}2\\
&\leq C_3 \sqrt{m} \calE_S(\tilde W).
\end{align*}
The first line follows by the triangle inequality, and the second since for vectors $a,b$, we have $\pnorm{ab^\top}F = \pnorm{a}2 \pnorm b2$.  The last line is by Cauchy--Schwarz, \eqref{eq:tilde.interlayer.upper.bounds}, and the definition of $\calE_S$, finishing the case $l=1$.  By substituting the definition of the gradient of the loss using the formula \eqref{eq:gradient.loss.at.tilde.w} we may similarly demonstrate the corresponding bounds for $l \geq 2$ with an application of Cauchy--Schwartz.
\end{proof}

\section{Proofs of Technical Lemmas}
In this section we go over the proofs of the technical lemmas that were introduced in Appendix \ref{appendix:key.lemma.proofs}.  In the course of proving these technical lemmas, we will need to introduce a handful of auxiliary lemmas, whose proofs we leave for Appendix \ref{appendix:auxiliary.lemma.proofs}.  Throughout this section, we continue to assume that $\theta = 1 / \Omega(L)$.
\label{appendix:technical.lemma.proofs}
\subsection{Proof of Lemma \ref{lemma:init.intermediate}: intermediate layers are bounded}
\label{appendix:lemma.init.intermediate.proof}
By Lemma \ref{lemma:init.weight.norm}, there is a constant $C_1$ such that with probability at least $1-\delta$, $\pnorm{W_l}2 \leq C_1$ for all $l=a,\dots, b$.  Therefore for each $r\geq 2$, we have
\[ \pnorm{I + \theta \tilde \Sigma_r W_r}2 \leq \pnorm{I}2 + \theta \pnorm{\tilde \Sigma_r}2 \pnorm{W_r}2 \leq  1 + \theta C_1.\]
The submultiplicative property of the spectral norm gives
\begin{align*}
&\pnorm{(I+\theta \tilde \Sigma_{b} W_b^\top) (I + \theta \tilde \Sigma_{b-1} W_{b-1}^\top)\cdot \ldots \cdot  (I + \theta \tilde \Sigma_a W_a^\top)}2 \\
&\leq \proddd r a b \pnorm{I + \theta \tilde \Sigma_r W_r^\top}2 \\
&\leq \l( 1 + \theta C_1\r)^L\\
&\leq \exp\l(C_1 \theta L\r).
\end{align*}
The result follows by the choice of scale $\theta = 1/\Omega(L)$ and taking $\theta$ small.

\subsection{Proof of Lemma \ref{lemma:init.lipschitz.lth.layer}: Lipschitz property with respect to input space at each layer}
\label{appendix:lemma:init.lipschitz.lth.layer}
Before beginning with the proof, we introduce the following claim that will allow us to develop a Lipschitz property with respect to the weights.  This was used in~\citet{cao2019} and~\citet{allenzhu2018}.
\begin{claim}
For arbitrary $u,y\in \R^{m_l}$, let $D(u)$ be the diagonal matrix with diagonal entries $[D(u)]_{j,j} = \ind( u_j \geq 0)$.  Then there exists another diagonal matrix $\check D(u)$ such that $\pnorm{D(u) + \check D(u)}2 \vee \pnorm{\check D(u)}2\leq 1$ and $\sigma(u) - \sigma(y) = \big(D(u) + \check D(u)\big) (u-y)$.
\label{claim:nhood.diagonal.defn}
\end{claim}
\begin{proof}[Proof of Claim \ref{claim:nhood.diagonal.defn}]
Simply define
\[ [\check D(u)]_{j,j} =\begin{cases}
[D(u) - D(y)] \f{y_j}{u_j-y_j} & u_j \neq y_j,\\
0 & u_j = y_j.\end{cases}\]
\end{proof}
\begin{proof}[Proof of Lemma \ref{lemma:init.lipschitz.lth.layer}]
We note that for any $x,y$, the matrix $|\Sigma_l(x) - \Sigma_l(y)|$ is zero everywhere except possibly the diagonal where it is either zero or one.  Therefore its spectral norm is uniformly bounded by $1$ for all $x,y$.  Using this, Lemma \ref{lemma:init.weight.norm} gives with probability at least $1-\delta/3$, for all $x,x' \in S^{d-1}$,
\begin{align*}
\pnorm{x_1 -  x_1'}2 &= \pnorm{(\Sigma_1(x_1) - \Sigma_1(x_1'))  W_1^\top(x - x')}2\\
&\leq \pnorm{\Sigma_1(x_1)- \Sigma_1(x_1')}2\pnorm{W_1}2 \pnorm{x- x'}2\\
&\leq 1 \cdot C \cdot \pnorm{x-  x'}2.
\end{align*}

For the case $L\geq l\geq 2$, we have residual links to analyze.  Using Claim \ref{claim:nhood.diagonal.defn} we can write
\[ \sigma(W_l^\top x_{l-1}) - \sigma(W_l^\top \hat x_{l-1}) = (\Sigma_l(x) + \check \Sigma_l(x)) W_l^\top(x_{l-1} - \hat x_{l-1})\]
for diagonal matrix $\check \Sigma_l$ satisfying $\pnorm{\check \Sigma_l(x)}2\leq 1$ and $\pnorm{\Sigma_l(x) + \check \Sigma_l(x)}2 \leq 1$. 
By Lemma \ref{lemma:init.intermediate}, we have with probability at least $1-\delta/3$, for all $2\leq l\leq L$ and all $x,x'\in S^{d-1}$,
\begin{align*}
\pnorm{x_l - x_l'}2 &\leq \pnorm{ I + \theta (\Sigma_l(x) + \check \Sigma_l(x)) W_l^\top }2\pnorm{x_{l-1} - x_{l-1}'}2\\
&\leq (1 + \theta C_0) \pnorm{x_{l-1} -  x_{l-1}'}2\\
&\leq \l( 1 + \f{C_0 \theta L}{L}\r)^L \cdot \pnorm{x -  x'}2\\
&\leq C_1 \pnorm{x - x'}2,
\end{align*}
since $\theta L$ is uniformly bounded from above.

The case $l = L+1$ follows as in the case $l=1$ by an application of Lemma \ref{lemma:init.weight.norm}, so that with probability at least $1-\delta/3$, $\pnorm{x_{L+1}'-x_{L+1}}2\leq C_2 \pnorm{x-x'}2$.  Putting the above three claims together, we get a constant $C_3$ such that with probability at least $1-\delta$, $\pnorm{x_l - x_l'}2 \leq C_3 \pnorm{x-x'}2$ for all $x,x'\in \calS^{d-1}$ and for all $l \in [L+1]$.

\end{proof}

\subsection{Proof of Lemma \ref{lemma:tau.nhood.lipschitz.weights}: local Lipschitz property with respect to weights and sparsity bound}
\label{appendix:lemma.tau.nhood.lipschitz.weights}
For this lemma, we need to introduce an auxiliary lemma that allows us to get control over the sparsity levels of the ReLU activation patterns.  Its proof can be found in Appendix \ref{appendix:lemma:init.count.indices.sd-1}.
\begin{lemma}
\label{lemma:init.count.indices.sd-1}
There are absolute constants $C, C'$ such that for any $\delta >0$, if 
$$m \geq C\l( \beta^{-1} \sqrt{d \log \f{1}{\beta \delta}} \vee d \log \f{mL}{\delta}\r),$$
then with probability at least $1-\delta$, the sets
\[ \calS_l(x,\beta) = \{ j\in [m_l]: |w_{l,j}^\top x_{l-1}| \leq \beta \},\, x\in S^{d-1},\, l\in [L+1],\]
satisfy $|\calS_l(\beta)| \leq C' m_l^{\f 32}\beta$ for all $x\in S^{d-1}$ and $l\in [L+1]$.
\end{lemma}
\begin{proof}[Proof of Lemma \ref{lemma:tau.nhood.lipschitz.weights}]
We begin with the Lipschitz property, and afterwards will show the sparsity bound.  Consider $l=1$.  Since $\hat x_1 = \sigma\l(\hat W_1^\top  x\r)$ and $\tilde x_1 =\sigma\l(\tilde W_1^\top x\r)$, by Claim \ref{claim:nhood.diagonal.defn}, for every $l$ there is a diagonal matrix $\check \Sigma_l(x)$ with $\pnorm{\check \Sigma_l(x)}2\leq 1$ and $\pnorm{\hat \Sigma_l(x) + \check \Sigma_l(x)}2 \leq 1$ such that 
\begin{align}
\nonumber
\pnorm{\hat x_1 - \tilde x_1}2 &= \pnorm{\l( \hat \Sigma_1(x) + \check \Sigma_1(x)\r) \l( \hat W_1^\top x - \tilde W_1^\top x\r)}2\\
\nonumber
&\leq \pnorm{\hat \Sigma_1(x) + \check \Sigma_1(x)}2 \pnorm{ \hat W_1 -  \tilde W_1}2 \pnorm{x}2\\
&\leq \pnorm{\hat W_1 - \tilde W_1}2.
\label{eq:hatx1-htildex1.bound}
\end{align}
For $l = 2, \dots, L$, we can write
\begin{align*}
\hat x_l - \tilde x_l &= \hat x_{l -1} + \theta \sigma\l ( \hat W_l^\top \hat x_{l-1} \r) - \tilde x_{l-1}  - \theta \sigma\l( \tilde W_l^\top \tilde x_{l-1}\r)\\
&= \l[ I + \theta \l(\hat \Sigma_l(x) + \check \Sigma_l(x) \r)\tilde W_l^\top \r]\l(\hat x_{l-1} - \tilde x_{l-1}\r) + \theta\l [ \hat \Sigma_l(x) + \check \Sigma_l(x)\r]\l(\hat W_l - \tilde W_l\r)^\top \hat x_{l-1}.
\end{align*}
Therefore, we have
\begin{align}
\nonumber
\pnorm{\hat x_l - \tilde x_l}2 &\leq \pnorm{I + \theta (\hat \Sigma_l(x) + \check \Sigma_l(x))\tilde  W_l^\top}2\pnorm{\hat x_{l-1} - \tilde x_{l-1}}2 + \theta\pnorm{ \hat \Sigma_l(x) + \check \Sigma_l(x)}2\pnorm{\hat W_l -\tilde  W_l}2\pnorm{ \hat x_{l-1}}2\\
&\leq \l(1 + C \theta \r) \pnorm{\hat x_{l-1} - \tilde x_{l-1}}2 + \theta \pnorm{\hat W_l -\tilde  W_l}2\pnorm{\hat x_{l-1}}2.\label{eq:hatxl.tildexl.induction}
\end{align}
We notice an easy induction will complete the proof. For the base case $l=2$, notice that $\pnorm{\hat x_1}2 \leq \pnorm{x_1}2 + \pnorm{\hat x_1 - x_1}2 \leq C + \tau \leq C'$, so that \eqref{eq:hatx1-htildex1.bound} and \eqref{eq:hatxl.tildexl.induction} give
\[ \pnorm{\hat x_2 - x_2}2 \leq \l(1 + C\theta\r) \pnorm{\hat W_1 - \tilde W_1}2 + C'\theta \pnorm{\hat W_2 - \tilde W_2}2 \leq C_4 \pnorm{\hat W_1 - \tilde W_1}2 + C_4 \theta \pnorm{\hat W_2 - \tilde W_2}2 .\]
Suppose by induction that there exists a constant $C$ such that $\pnorm{\hat x_{l-1} - x_{l-1}}2 \leq C_5 \pnorm{\hat W_1 - \tilde W_1}2 + C_5 \theta \summ r {l-1} \pnorm{\hat W_r - \tilde W_r}2$.  Then as in the base case, $\pnorm{\hat x_{l-1}}2 \leq C'$, so that \eqref{eq:hatxl.tildexl.induction} gives for all $l=2,\dots, L$,
\begin{align*}
\pnorm{\hat x_l - \tilde x_l}2 &\leq \l( 1 + C \theta\r) C \l[  C_5 \pnorm{\hat W_1 - \tilde W_1}2 + C_5 \theta \summ r {l-1} \pnorm{\hat W_r - \tilde W_r}2\r]  + C'\theta \pnorm{\hat W_l - \tilde W_l}2\\
&\leq C_6 \pnorm{\hat W_1 - \tilde W_1}2 + C_6 \theta \summ r l \pnorm{\hat W_r - \tilde W_r}2.
\end{align*}
Finally, the case $l = L+1$ follows similarly to the case $l\leq L$, as
\begin{align*}
\pnorm{\hat x_{L+1} - \tilde x_{L+1}} 2&= \pnorm{\l( \hat \Sigma_{L+1}(x) + \check \Sigma_{L+1}(x)\r) \l( \hat W_{L+1}^\top \hat x_L - \tilde W_{L+1}^\top\tilde  x_L\r)}2\\
&\leq C\pnorm{\hat W_{L+1} -\tilde W_{L+1}}2 + C'\pnorm{\hat x_L - \tilde x_L}2.
\end{align*}

The bound for the sparsity levels of $\tilde \Sigma_l(x) - \hat \Sigma_l(x)$ follows the same proof as Lemma B.5 in \citet{cao2019} with an application of our Lemma \ref{lemma:init.count.indices.sd-1}.  Sketching this proof, we note that it suffices to prove a bound for $\pnorm{\hat \Sigma_l(x) - \Sigma_l(x)}0$, use the same proof for $\pnorm{\tilde \Sigma_l(x) - \Sigma_l(x)}0$ and then use triangle inequality to get the final result.  We write
\begin{align*}
\pnorm{\hat \Sigma_l(x) - \Sigma_l(x)}0 
&= s_{l}^{(1)}(\beta) + s_{l}^{(2)}(\beta),
\end{align*}
where
\begin{align*}
s_l^{(1)}(\beta) &= | \{ j\in \calS_l(x,\beta): (\hat w_{l,j}^\top \hat x_{l-1} ) \cdot ( w_{l,j}^\top x_{l-1}) < 0 \} |,\\
s_l^{(2)}(\beta) &= | \{ j\in \calS_l^c(x,\beta): (\hat w_{l,j}^\top \hat x_{l-1} ) \cdot ( w_{l,j}^\top x_{l-1}) < 0 \} |,
\end{align*}
which leads to 
\begin{equation}
\pnorm{\hat \Sigma_l(x) - \Sigma_l(x)}0 \leq Cm^{\f 3 2} \beta + C_5\tau^2 \beta^{-2}.
\label{eq:Sigma_ell.sd-1.bound}
\nonumber
\end{equation}
The choice of $\beta = m_l^{-\f 12} \tau^{\f 23}$ completes the proof. 
\end{proof}

\subsection{Proof of Lemma \ref{lemma:tau.nhood.sparse.rhs}: behavior of network output in $\calW(W^{(0)}, \tau)$ when acting on sparse vectors}
\label{appendix:lemma.tau.nhood.sparse.rhs}
This technical lemma will require two auxiliary lemmas before we may begin the proof.  Their proofs are left for Appendix \ref{appendix:lemma:lastlayer.init.sparse.bothsides} and \ref{appendix:lemma:lastlayer.init.sparse.rhs}.
\begin{lemma}
\label{lemma:lastlayer.init.sparse.bothsides}
Consider the function $g_l: \R^{m_l} \times \R^{m_{L+1}}\to \R$ defined by
\[ g_l(a,b) := b^\top W_{L+1}^\top \xi_l a,.\]
where $\xi_{l} \in \R^{m_{L}\times m_l}$, and $l \geq 2$.  Suppose that with probability at least $1-\delta/2$, $\pnorm{\xi_l}2 \leq C$ holds for all $\xi_l$, $l=2,\dots, L$.  If $s \log m = \bigOmega{ C \log (L/\delta)}$, then there is a constant $C_0>0$ such that probability at least $1-\delta$, for all $l$, 
\[ \sup_{\pnorm{a}2=\pnorm{b}2=1,\ \pnorm{a}0,\pnorm{b}0\leq s} |g_l(a,b)| \leq C_0 \sqrt{\f 1 m s \log m}.\]
\end{lemma}

\begin{lemma}
\label{lemma:lastlayer.init.sparse.rhs}
Consider the function $g_l : \R^{m_l} \to \R$ defined by
\[ g_l(a) := v^\top \Sigma_{L+1}(x)^\top W_{L+1}^\top \xi_{l}a,\]
where $\xi_{l} \in \R^{m_{L}\times m_l}$ and $l \geq 2$.  Assume that with probability at least $1-\delta$, $\pnorm{\xi_l}2 \leq C_0$ for all $l$.  Then provided $s \log m = \bigOmega{ \log (L/\delta)}$, we have with probability at least $1-\delta$, for all $l$,
\[ \sup_{\pnorm{a}2 =1,\ \pnorm{a}0\leq s} |g_l(a)| \leq C_1 \sqrt{s \log m}.\]
\end{lemma}

With these lemmas in place, we can prove Lemma \ref{lemma:tau.nhood.sparse.rhs}.
\begin{proof}[Proof of Lemma \ref{lemma:tau.nhood.sparse.rhs}]
By definition, $g_l(a,x) = v^\top \tilde H_l^{L+1} a$.  First: since $\pnorm{\tilde W_l - W_l}2\leq \tau$, there is an absolute constant $C_2 >0$ such that with high probability, $\pnorm{\tilde W_l}2 \leq C_2$ for all $l$.  Therefore, we have with high probability for all $x\in S^{d-1}$, all $l$, and all $a$ considered,
\begin{align}
\pnorm{\tilde H_l^L }2 &\leq \l[ \proddd r l L \pnorm{I + \theta \tilde \Sigma_r(x) \tilde W_r^\top }2\r] \pnorm{a}2\leq \l( 1 + \theta \cdot 1 \cdot C_2 \r)^L \cdot 1 \leq C_3,
\label{eq:tildehlL.bound}
\end{align}
by our choice of $\theta$.  We proceed by bounding $g_l$ by a sum of four terms:
\begin{align*}
&|g_l(a,x)| \leq a \leq \l|v^\top \l( \tilde \Sigma_{L+1}(x) - \Sigma_{L+1}(x)\r) \tilde W_{L+1}^\top \tilde H_l^L a\r| + \l| v^\top \Sigma_{L+1}(x) \tilde W_{L+1}^\top \tilde H_l^L a\r|\\
&\leq \l|v^\top \l( \tilde \Sigma_{L+1}(x) - \Sigma_{L+1}(x)\r)\l( \tilde W_{L+1}^\top -W_{L+1}^\top\r) \tilde H_l^L a\r| + \l|v^\top \l( \tilde \Sigma_{L+1}(x) - \Sigma_{L+1}(x)\r) W_{L+1}^\top \tilde H_l^L a\r| \\
&\qquad+\l| v^\top \Sigma_{L+1}(x) \l(\tilde W_{L+1}^\top - W_{L+1}^\top\r) \tilde H_l^L a\r| + \l| v^\top \Sigma_{L+1}(x) W_{L+1}^\top \tilde H_l^L a\r|. \\
\end{align*}
For the first term, we can write
\begin{align*}
 &\l|v^\top \l( \tilde \Sigma_{L+1}(x) - \Sigma_{L+1}(x)\r)\l( \tilde W_{L+1}^\top -W_{L+1}^\top\r) \tilde H_l^L\r| \\
 &\qquad\qquad\leq \pnorm{v}2 \pnorm{\l( \tilde \Sigma_{L+1}(x) - \Sigma_{L+1}(x)\r)\l( \tilde W_{L+1}^\top -W_{L+1}^\top\r) H_l^L a}2 \\
&\qquad\qquad\leq C \sqrt{m} \pnorm{ \tilde \Sigma_{L+1}(x) - \Sigma_{L+1}(x)}2 \pnorm{\tilde W_{L+1} - W_{L+1}}2 \pnorm{\tilde H_l^L a}2\\
&\qquad\qquad\leq C' \tau \sqrt{m},
\end{align*}
where we have used Cauchy--Schwarz in the first line, properties of the spectral norm in the second, and \eqref{eq:tildehlL.bound} in the third.  A similar calculation shows
\begin{align*}
\l |v^\top \Sigma_{L+1} \l( \tilde W_{l+1}^\top - W_{L+1}^\top\r) \tilde H_l^L\r| &\leq \pnorm{v}2 \pnorm{\Sigma_{L+1} \l( \tilde W_{L+1}^\top - W_{L+1}^\top \r) \tilde H_l^L}2\\
&\leq C \tau\sqrt m.
\end{align*}
For the second and fourth terms, we use Lemmas \ref{lemma:lastlayer.init.sparse.bothsides} and \ref{lemma:lastlayer.init.sparse.rhs}.  Let $\check b^\top = v^\top \l( \tilde \Sigma_{L+1}(x) - \Sigma_{L+1}(x)\r)$.  Then it is clear that $\pnorm{\check b}0\leq s$ and $\pnorm{\check b}2 \leq \sqrt{m}$ (in fact, $\pnorm{\check b}2\leq \sqrt{s}$, but this doesn't matter since the fourth term dominates the second term).   Thus applying Lemma \ref{lemma:lastlayer.init.sparse.bothsides} to $b = \check b / \pnorm{\check b}2$,
\begin{align*}
|v^\top \l( \tilde \Sigma_{L+1}(x) - \Sigma_{L+1}(x)\r) W_{L+1}^\top \tilde H_l^L a| &\leq C \sqrt{m} \cdot \sqrt{\f s m \log m}\\
&\leq C \sqrt{s \log m}.
\end{align*}
For the fourth term, we can directly apply Lemma \ref{lemma:lastlayer.init.sparse.rhs} to get another term $\propto \sqrt{s \log m}$.  
\end{proof}

\subsection{Proof of Lemma \ref{lemma:gen.lastlayer.lowerbound}}
\label{appendix:lemma.gen.lastlayer.lowerbound}
This lemma is the key to the sublinear dependence on $L$ for the required width for the generalization result.  Essential to its proof is the following proposition which states that there is a linear separability condition at each layer due to Assumption \ref{assumption:separability} with only a logarithmic dependence on the depth $L$.  In fact, we only need linear separability at the second-to-last layer for the proof of Lemma \ref{lemma:gen.lastlayer.lowerbound}.
\begin{proposition}
\label{proposition:gen.linsep}
Suppose $m\geq C \gamma^{-2} \l( d \log \f 1 \gamma + \log \f L \delta\r)$ for some large constant $C$. 
Then there exists $\alpha \in S^{m_L-1}$ such that with probability at least $1-\delta$, for all $l=1,\dots, L$, we have
\[ y \ip{\alpha}{x_l} \geq \gamma/2.\]
\end{proposition}
\begin{proof}[Proof of Proposition \ref{proposition:gen.linsep}]
We recall that Assumption \ref{assumption:separability} implies that there exists $c(\overline{\mathbf{u}})$ with $\pnorm{c(u)}\infty \leq 1$ such that $f(x) = \int_{\R^d} c(u) \sigma(u^\top x) p(u) du$ satisfies $y \cdot f(x) \geq \gamma$ for all $(x,y)\in \supp(\calD)$.  Following Lemma C.1 in~\citet{cao2019}, if we define
\[ \alpha := \sqrt{\f {1}{m_1}} \cdot \l( c\l( \sqrt{\f {m_1}2} w_{1,1}\r), \dots, c\l( \sqrt{\f{m_1}2} w_{1,m_1}\r)\r),\]
then $\alpha = \alpha'/\pnorm{\alpha'}2 \in S^{m_1-1}$ satisfies $y\cdot \alpha^\top x_1 \geq \f \gamma 2$ for all $(x,y)\in \supp \calD$.  

We now show that the $l$-th layer activations $x_l$ are linearly separable using $\alpha$.  We can write, for $l=2,\dots, L$,
\begin{align}
\nonumber
\ip{\alpha}{x_l} &= \ip{\alpha}{(I + \theta \Sigma_l(x) W_l^\top) x_{l-1}}\\
\label{eq:linsep.alphalconst.def}
&=\ip{\alpha}{x_1} + \theta \summm {l'}2 l \ip{\alpha}{\Sigma_{l'}(x) W_{l'}^\top x_{l'-1}}.
\end{align}
Since $\ip{\alpha}{\Sigma_l(x) W_l^\top x_{l-1}} = \summ k {m_{l}} \sqrt{\f 1{m_1}} c\l( \sqrt{\f {m_1}2} w_{1,k} \r) \cdot \sigma(w_{l,k}^\top x_{l-1})$ and $\pnorm{c(\cdot)}\infty \leq 1$, we have for every $l \geq 2$,
\begin{align}
 - \summ k {m_l} \sqrt{\f 1 {m_1}} \l|  w_{l,k}^\top x_{l-1} \r|\leq \ip{\alpha}{\Sigma_l(x) W_l^\top x_{l-1}}  \leq  \summ k {m_l} \sqrt{\f 1 {m_1}} \l|  w_{l,k}^\top x_{l-1} \r|.
\label{eq:alpha.lth.layer}
\end{align}
Thus it suffices to find an upper bound for the term on the r.h.s. of \eqref{eq:alpha.lth.layer}.  
Since we have
\begin{align*}
\E\l|w_{l,k}^\top x_{l-1}\r| = \sqrt{\f 2 \pi} \sqrt{\f 2 {m_1}} \pnorm{x_{l-1}}2 \leq C_2 m^{-\f 12},
\end{align*}
we can apply Hoeffding inequality to get absolute constants $C_4,C_5 >0$ such that for fixed $x$ and $l$, we have with probability at least $1-\delta$,
\begin{align*}
 \summ k {m_l} \sqrt{\f 1 {m_1}} \l|w_{l,k}^\top x_{l-1} \r|  &\leq \summ k {m_l} \sqrt{\f{1}{m}} C_2 m^{-\f 12} + C_4 \sqrt{\f 1 m \log \f 1 \delta}\\
 &\leq C_5 + C_4 \sqrt{\f 1 m \log \f 1 \delta}.
\end{align*}
Take a $\f 12$-net $\calN$ of $S^{d-1}$ so that $|\calN| \leq 5^d$ and every $x\in S^{d-1}$ has $\hat x\in \calN$ with $\pnorm{x-\hat x}2\leq \f 12$.  Then, provided $m \geq C d \log \f{L}\delta$, there is a constant $C_6 > 0$ such that we have with probability at least $1-\delta$, for all $\hat x\in \calN$ and all $l \leq L$,
\begin{align*}
\summ k {m_l} \sqrt{\f 1 {m_1}} \l|w_{l,k}^\top \hat x_{l-1} \r|  &\leq C_6.
\end{align*} 
By \eqref{eq:alpha.lth.layer}, this means for all $\hat x \in \calN$ and $l$, $-C_6 \leq \ip{\alpha}{\Sigma_l(\hat x) W_l^\top \hat x_{l-1}} \leq  C_6$.  We can lift this to hold over $S^{d-1}$ by using Lemma \ref{lemma:init.lipschitz.lth.layer}:  for arbitrary $x\in S^{d-1}$ we have
\begin{align*}
\l|\ip{\alpha}{\Sigma_l(x) W_l^\top x_l}\r| &\leq \l|\ip{\alpha}{\Sigma_l(x) W_l^\top (x_l - \hat x_l)}\r| +\l|\ip{\alpha}{\Sigma_l(x) W_l^\top \hat x_l}\r|\\
&\leq \pnorm{\tilde \alpha_{l}}2 \pnorm{\Sigma_l(x)}2 \pnorm{W_l}2 \pnorm{x_l - \hat x_l}2 + C_6\\
&\leq C_7,
\end{align*}
so that with probability at least $1-\delta$, for all $l\leq L$ and all $x\in S^{d-1}$, we have
\[ -C_7 \leq \ip{\alpha}{\Sigma_l(x) W_l^\top \hat x_{l-1}} \leq C_7.\]
Substituting the above into \eqref{eq:linsep.alphalconst.def}, we get
\[ \begin{cases}
\ip{\alpha}{x_l} \geq \ip{\alpha}{x_1} - \theta L C_7,\\
-\ip{\alpha}{x_l} \geq -\ip{\alpha}{x_1} - \theta L C_7.
\end{cases} \]
Considering the cases $y=\pm 1$ we thus get with probability at least $1-\delta$ for all $l$ and $(x,y)\in \supp \calD$,
\[\begin{cases}
y\ip{\alpha}{x_l} \geq y\ip{\alpha}{x_1} - \theta L C_7 \geq \f \gamma 2 - \theta L C_7,&y=1,\\
y\ip{\alpha}{x_l} \geq y\ip{\alpha}{ x_1} - \theta L C_7 \geq \f \gamma 2 - \theta L C_7,&y=-1.
\end{cases}\]
Thus taking $\theta$ small enough so that $\theta L \leq \gamma C_7^{-1}/4$ completes the proof.
\end{proof}

With Proposition \ref{proposition:gen.linsep} in hand, we can prove Lemma \ref{lemma:gen.lastlayer.lowerbound}.
\begin{proof}[Proof of Lemma \ref{lemma:gen.lastlayer.lowerbound}]
By Proposition \ref{proposition:gen.linsep}, there exists $\alpha_L\in S^{m_L-1}$ such that with probability at least $1-\delta$, $y\ip{\alpha_L}{x_L} \geq \gamma/4$ for all $(x,y)\in \supp(\calD)$.  In particular, since $a$ is non-negative, this implies for all $i$,
\begin{equation}
\ip{a(x_i,y_i) \cdot y_i \cdot x_{L,i} }{\alpha_L} = a(x_i,y_i)\cdot y_i \ip{ x_{L,i}}{\alpha_{L}} \geq a(x_i,y_i) y_i \gamma/4.
\label{eq:gen.a(x_i,y_i)}
\end{equation}

Since $\E[\sigma'(w_{L+1,j}^\top x_{L,i}) | x_{L,i}] = \f 1 2$, by Hoeffding inequality, with probability at least $1-\delta/2$, for all $i=1,\dots, n$, we have
\begin{equation}
    \f 1 {m_{L+1}} \summ j {m_{L+1}} \sigma'(w_{L+1,j}^\top x_{L,i}) \geq \f 1 2 - C_1 \sqrt{\f 1 {m_{L+1}} \log(n/\delta)} \geq \f {49}{100}.
    \label{eq:gen.relu.concentrates.at.half}
\end{equation}
Therefore, we can bound
\begin{align*}
&\summ j {m_{L+1}} \pnorm{\f 1 n \summ i n \l[ a(x_i,y_i) \cdot y_i \cdot \sigma'(w_{L+1,j}^\top x_{L,i}) \cdot x_{L,i} \r]}2^2\\
&\geq m_{L+1} \pnorm{\f 1 {m_{L+1}} \summ j {m_{L+1}} \f 1 n \summ i n \l[ a(x_i,y_i) \cdot y_i \cdot \sigma'(w_{L+1,j}^\top x_{L,i}) \cdot x_{L,i} \r]}2^2\\
&= m_{L+1} \pnorm{\f 1 n \summ i n \l[ a(x_i,y_i)\cdot y_i \cdot x_{L,i} \f 1 {m_{L+1}} \summ j {m_{L+1}}   \sigma'(w_{L+1,j}^\top x_{L,i}) \r]}2^2\\
&\geq m_{L+1} \ip{ \f 1 n \summ i n a(x_i, y_i)\cdot y_i \cdot x_{L,i} \cdot \f 1 {m_{L+1}} \summ j {m_{L+1}} \sigma'(w_{L+1,j}^\top x_{L,i})}{\alpha_L}^2\\
&= m_{L+1} \l( \f 1 n \summ i n a(x_i,y_i)\cdot y_i \cdot \f 1 {m_{L+1}} \summ j {m_{L+1}} \sigma'(w_{L+1,j}^\top x_{L,i}) \cdot \ip{x_{L,i}}{\alpha_L}\r)^2\\
&\geq \l(\f {49}{100}\r)^2 m_{L+1} \l( \f 1 n \summ i n a(x_i,y_i)\r)^2 \cdot \f{\gamma^2}{4^2}\\
&\geq \f 1 {67} m_{L+1} \cdot \gamma^2 \l( \f 1 n \summ i n a(x_i,y_i)\r)^2.
\end{align*}
The first inequality above follows by Jensen inequality.  The second inequality follows by Cauchy--Schwarz and since $\pnorm{\alpha_L}2=1$.  The third inequality follows with an application of \eqref{eq:gen.a(x_i,y_i)} and \eqref{eq:gen.relu.concentrates.at.half}, and the final inequality by arithmetic. 
\end{proof}

\section{Proofs of Auxiliary Lemmas}
\label{appendix:auxiliary.lemma.proofs}

\subsection{Proof of Lemma \ref{lemma:init.count.indices.sd-1}}
\label{appendix:lemma:init.count.indices.sd-1}
\begin{proof}
By following a proof similar to that of Lemma A.8 in \citet{cao2019}, one can easily prove the following claim:
\begin{claim}
For $v\in \R^{m_{l-1}}$, $\beta >0$, and $l \in [L+1]$ define
\begin{equation}
\calS_l(v, \beta) := \{ j\in [m_l]: |w_{l,j}^\top v| \leq \beta \}.
\end{equation} 
Suppose that there is an absolute constant $\xi \in (0,1)$ such that for any $\delta >0$ we have with probability at least $1-\delta/2$, $\pnorm{v}2 \geq \xi$ for all $v\in \calV$ for some finite set $\calV\subset \R^{m_{l-1}}$.  Then there exist absolute constants $C,C'>0$ such that if $m\geq C\beta^{-1}\sqrt{ \log( 4 |\calV|/\delta)}$, then with probability at least $1-\delta$, we have $|\calS_l(v, \beta)| \leq C' m_l^{3/2} \beta$ for all $v\in \calV$.
\label{claim:init.count.indices}
\end{claim}

By Lemmas \ref{lemma:hidden.and.interlayer.activations.bounded} and \ref{lemma:init.weight.norm}, with probability at least $1-\delta/3$, we have $\pnorm{x_{l-1}}2\geq C$ and $\pnorm{w_{l,j}}2 \leq C_1$ for all $x\in S^{d-1}$, $l \in [L+1]$, and $j\in [m_l]$.  By Lemma \ref{lemma:init.lipschitz.lth.layer}, with probability at least $1-\delta/3$, we have $\pnorm{x_l - x_l'}2\leq C_2 \pnorm{x-x'}2$ for all $x,x'\in S^{d-1}$.   By taking $\calV$ to be the $\beta/(C_1C_2)$-net $\calN(S^{d-1}, \beta/(C_1C_2))$, since $|\calN| \leq (4C_1C_2/\beta)^d$,
the assumption that $m\geq C \beta^{-1} \sqrt{d \log (1/(\beta \delta))}$ allows us to apply Lemma \ref{claim:init.count.indices} to get that with probability at least $1-\delta/3$, we have $|\calS_l(\hat x, 2\beta)|\leq 2 C' m_l^{\f 32} \beta$ for all $l$ and $\hat x\in \calN$.  For arbitrary $x\in S^{d-1}$, there exists $\hat x\in \calN$ with $\pnorm{x-\hat x}2 \leq \beta/(C_1C_2)$.  Thus, we have
\begin{align*}
|w_{l,j}^\top x_{l-1}| &\leq |w_{l,j}^\top \hat x_{l-1}| + |w_{l,j}^\top(x_{l-1} - \hat x_{l-1})|\\
&\leq \beta + \pnorm{w_{l,j}}2 \pnorm{x_{l-1} - \hat x_{l-1}}2\\
&\leq \beta + C_1 \cdot C_2 \pnorm{x - \hat x}2 \\
&\leq 2\beta,
\end{align*}
i.e. $\calS_l(x,\beta) \subset \calS_l(\hat x, 2 \beta)$.  Therefore $|\calS_l(x, \beta)| \leq |\calS_l(\hat x, 2\beta)| \leq 2 C' m_l^{\f 32} \beta$, as desired.
\end{proof}

\subsection{Proof of Lemma \ref{lemma:lastlayer.init.sparse.bothsides}}
\label{appendix:lemma:lastlayer.init.sparse.bothsides}
\begin{proof}
The $j$-th row of $W_{L+1}^\top \xi_l a$ has distribution $w_{L+1,j}^\top \xi_l a \sim N\l(0, \f 2 {m_{L+1}} \pnorm{\xi_l a}2^2\r)$, and hence $g_l(a,b) \sim N\l(0, \f 2 {m_l}\pnorm{\xi_l a}2^2\r)$.
Since $\pnorm{\xi_l}2\leq C_0$ for all $l$ with high probability, it is clear that $\pnorm{\xi_l a}2^2 \leq C_0^2$.  Thus 
applying Hoeffding inequality gives a constant $C_3>0$ such that we have for fixed $a$ and $b$, with probability at least $1-\delta$,
\begin{equation}
|b^\top W_{L+1}^\top \xi_l a| \leq C_3 \sqrt{ \f 1 {m_{L+1}} \log \f 1 \delta}.
\label{eq:btop.wlplus1top}
\end{equation}
Let $\calM_a$ be a fixed subspace of $\R^{m_l}$ with sparsity $s$, and let $\calN_a(\calM, 1/4)$ be a $1/4$-net covering $\calM_a$.  There are $\binom {m_{l}}s$ choices of such $\calM_a$.  Let $\calN_a = \cup_{\calM_a} \calN_a(\calM_a, 1/4)$ be the union of such spaces.  By Lemma 5.2 in  \citet{vershynin}, for $s$ larger than e.g. 15, we have
\[ |\calN_a| \leq \binom {m_l}s 9^s \leq m_{l}^s.\]
Similarly consider subspace $\calM_b \subset \R^{m_{L+1}}$ with sparsity level $s$ and let $\calN_b(\calM_b, 1/4)$ be a $1/4$-net of $\R^{m_{L+1}}$ with sparsity level $s$ and define $\calN_b = \cup_{\calM_b} \calN_b(\calM_b, 1/4)$, so that $|\calN_b| \leq m_{L+1}^s$.
We apply \eqref{eq:btop.wlplus1top} to every $\hat a\in \calN_a$ and $\hat b\in \calN_b$ and use a union bound to get a constant $C_4>0$ such that with probability at least $1-\delta$, for all $\hat a\in \calN_a, \hat b\in \calN_b$, and all $l$,
\begin{align*}
|\hat b^\top W_{L+1}^\top \xi_l \hat a| &\leq C_3 \sqrt{\f 1 {m_{L+1}} \log \f {|\calN_a| \cdot |\calN_b|\cdot L}\delta}\\
&\leq C_3\sqrt{\f 1 {m_{L+1}} \log \f {m_{L+1}^s \cdot m_l^s\cdot L}\delta}\\
&= C_3 \sqrt{ \f 1 {m_{L+1}} \l( s \log (m_{L+1} m_l) + \log \f L \delta\r)}\\
&\leq C_4 \sqrt{ \f s {m_{L+1}} \log m }. &\l(s \log m = \Omega\l( \log \f L \delta\r)\r)\\
\end{align*}
For arbitrary $a\in S^{m_l-1}$ and $b\in S^{m_{L+1}-1}$ with $\pnorm{a}0,\pnorm{b}0\leq s$, there are $\hat a\in \calN_a$ and $\hat b\in \calN_b$ with $\pnorm{a - \hat a}2,\ \pnorm{b - \hat b}2\leq 1/4$.  Note that $g$ is linear in $a$ and $b$.  Triangle inequality gives
\begin{align}
\nonumber
|g_l(a,b)| &\leq |g_l(\hat a, \hat b)| + |g_l(a, b) - g_l(\hat a, \hat b)| \\
&\leq C_3 \sqrt{\f s {m_{L+1}} \log {m_{L+1}}} + |g_l(a, b) - g_l(\hat a,  b)| + |g_l(\hat a, \hat b) - g_l(\hat a, b)| 
\label{eq:gab}
\end{align}
We have for any $\hat a$,
\begin{align}
\nonumber
|g_l(\hat a, \hat b) - g_l(\hat a, b)|&= \pnorm{b-\hat b}2 \l|g_l\l(\hat a, \f{b-\hat b}{\pnorm{b-\hat b}2}\r)\r|\\
\label{eq:ghatab}
&\leq \f 1 4 \sup_{\pnorm{b'}2=\pnorm{a}2=1,\ \pnorm{a}0,\pnorm{b'}0\leq s} \l|g_l\l(a, b'\r)\r|.
\end{align}
Similarly, 
\begin{equation}
|g_l(a,b) - g_l(\hat a, b)| \leq \f 1 4 \sup_{\pnorm{b}2=\pnorm{a}2=1,\ \pnorm{a}0,\pnorm{b}0\leq s} \l|g_l\l(a, b\r)\r|.
\label{eq:gahatb}
\end{equation}
Taking supremum over the left hand side of \eqref{eq:gab} and using the bounds in \eqref{eq:ghatab} and \eqref{eq:gahatb} completes the proof.
\end{proof}

\subsection{Proof of Lemma \ref{lemma:lastlayer.init.sparse.rhs}}
\label{appendix:lemma:lastlayer.init.sparse.rhs}
\begin{proof}
We notice that since $v = (1, \dots, 1, -1, \dots, -1)^\top$,  we can write $g_l(a)$ as a sum of independent random variables in the following form:
\begin{align*}
    g_l(a) 
    &= \sqrt{m_{L+1}} \summ j {m_{L+1}/2} \f 1 {\sqrt{m_{L+1}}}\l[ \sigma(w_{L+1,j}^\top \xi_{l+1}a) - \sigma(w_{L+1,j+m_{L+1}/2}^\top \xi_{l+1} a)\r].
\end{align*}
Since $\pnorm{\xi_{l+1}a}2$ is uniformly bounded by a constant, Hoeffding inequality yields a constant $C_3>0$ such that for fixed $a$, with probability at least $1-\delta$, we have
\[ g_l(a) \leq C_3\sqrt{m} \sqrt{\f 1 {m} \log \f 1 \delta}.\]
Let $\calM$ be a fixed subspace of $\R^{m_l}$ with sparsity $s$, and let $\calN = \cup_\calM \calN(\calM, 1/2)$ be the union of all $1/2$-nets covering each $\calM$ so that $|\calN| \leq m_l^s$.  Using a union bound over all $\hat a\in \calN$ and $l$, we get that with probability at least $1-\delta$, for all $\hat a\in \calN$ and all $l \leq L$,
\begin{align*}
g_l(\hat a) \leq C_3 \sqrt{m} \cdot \sqrt{\f 1 m \log \f{|\calN| \cdot L}{\delta}} \leq C_5 \sqrt{s \log m}.
\end{align*}
For arbitrary $a\in S^{m_l-1}$ satisfying $\pnorm{a}0\leq s$, there is $\hat a\in \calN$ with $\pnorm{a-\hat a}2\leq 1/2$.  Since $g$ is linear,
\begin{align}
|g_l(a)| &\leq |g_l(\hat a)| + |g_l(a - \hat a)| \leq C_5 \sqrt{s \log m} + |g_l(a - \hat a)|.
\label{eq:ga}
\end{align}
For the second term, we have
\begin{align*}
|g_l(a - \hat a)| = \pnorm{a - \hat a}2 \l| g_l\l( \f{a-\hat a}{\pnorm{a - \hat a}2} \r) \r|\leq \f 1 2 \sup_{\pnorm{a}2=1,\ \pnorm{a}0\leq s} |g_l(a)|.
\end{align*}
Substituting this into \eqref{eq:ga} and taking supremums completes the proof.
\end{proof}

\bibliographystyle{ims}
\bibliography{references}

\begin{thebibliography}{28}
\expandafter\ifx\csname natexlab\endcsname\relax\def\natexlab#1{#1}\fi
\expandafter\ifx\csname url\endcsname\relax
  \def\url#1{\texttt{#1}}\fi
\expandafter\ifx\csname urlprefix\endcsname\relax\def\urlprefix{URL }\fi

\bibitem[{Allen{-}Zhu et~al.(2019)Allen{-}Zhu, Li and Song}]{allenzhu2018}
\textsc{Allen{-}Zhu, Z.}, \textsc{Li, Y.} and \textsc{Song, Z.} (2019).
\newblock A convergence theory for deep learning via over-parameterization.
\newblock In \textit{International Conference on Machine Learning}.

\bibitem[{Arora et~al.(2019)Arora, Du, Hu, Li and
  Wang}]{aroradu2019-finegrained}
\textsc{Arora, S.}, \textsc{Du, S.~S.}, \textsc{Hu, W.}, \textsc{Li, Z.} and
  \textsc{Wang, R.} (2019).
\newblock Fine-grained analysis of optimization and generalization for
  overparameterized two-layer neural networks.
\newblock In \textit{International Conference on Machine Learning}.

\bibitem[{Arora et~al.(2018)Arora, Ge, Neyshabur and
  Zhang}]{arora2018-compression}
\textsc{Arora, S.}, \textsc{Ge, R.}, \textsc{Neyshabur, B.} and \textsc{Zhang,
  Y.} (2018).
\newblock Stronger generalization bounds for deep nets via a compression
  approach.
\newblock In \textit{{ICML}}, vol.~80 of \textit{Proceedings of Machine
  Learning Research}. {PMLR}.

\bibitem[{Bartlett et~al.(2017)Bartlett, Foster and Telgarsky}]{bartlett2017}
\textsc{Bartlett, P.~L.}, \textsc{Foster, D.~J.} and \textsc{Telgarsky, M.~J.}
  (2017).
\newblock Spectrally-normalized margin bounds for neural networks.
\newblock In \textit{{Conference on Neural Information Processing Systems}}.

\bibitem[{Cao and Gu(2019{\natexlab{a}})}]{cao2019wide}
\textsc{Cao, Y.} and \textsc{Gu, Q.} (2019{\natexlab{a}}).
\newblock Generalization bounds of stochastic gradient descent for wide and
  deep neural networks.
\newblock In \textit{{Conference on Neural Information Processing Systems}}.

\bibitem[{Cao and Gu(2019{\natexlab{b}})}]{cao2019}
\textsc{Cao, Y.} and \textsc{Gu, Q.} (2019{\natexlab{b}}).
\newblock Generalization error bounds of gradient descent for learning
  over-parameterized deep relu networks.
\newblock \textit{arXiv preprint} \textbf{arXiv:1902.01384}.

\bibitem[{Choi et~al.(2019)Choi, Seo, Shin, Byun, Kersner, Kim, Kim and
  Ha}]{choi2019}
\textsc{Choi, S.}, \textsc{Seo, S.}, \textsc{Shin, B.}, \textsc{Byun, H.},
  \textsc{Kersner, M.}, \textsc{Kim, B.}, \textsc{Kim, D.} and \textsc{Ha, S.}
  (2019).
\newblock Temporal convolution for real-time keyword spotting on mobile
  devices.
\newblock \textit{arXiv preprint} \textbf{arXiv:1904.03814}.

\bibitem[{Du et~al.(2019{\natexlab{a}})Du, Lee, Li, Wang and
  Zhai}]{du2018-deep}
\textsc{Du, S.~S.}, \textsc{Lee, J.~D.}, \textsc{Li, H.}, \textsc{Wang, L.} and
  \textsc{Zhai, X.} (2019{\natexlab{a}}).
\newblock Gradient descent finds global minima of deep neural networks.
\newblock In \textit{International Conference on Machine Learning}.

\bibitem[{Du et~al.(2019{\natexlab{b}})Du, Zhai, P{\'{o}}czos and
  Singh}]{du2018-1layer}
\textsc{Du, S.~S.}, \textsc{Zhai, X.}, \textsc{P{\'{o}}czos, B.} and
  \textsc{Singh, A.} (2019{\natexlab{b}}).
\newblock Gradient descent provably optimizes over-parameterized neural
  networks.
\newblock In \textit{International Conference on Learning Representations}.

\bibitem[{Dziugaite and Roy(2017)}]{dziugaiteroy2017}
\textsc{Dziugaite, G.~K.} and \textsc{Roy, D.~M.} (2017).
\newblock Computing nonvacuous generalization bounds for deep (stochastic)
  neural networks with many more parameters than training data.
\newblock In \textit{Proceedings of the Thirty-Third Conference on Uncertainty
  in Artificial Intelligence, {UAI} 2017, Sydney, Australia, August 11-15,
  2017}.

\bibitem[{E et~al.(2019)E, Ma, Wang and Wu}]{weinan2019}
\textsc{E, W.}, \textsc{Ma, C.}, \textsc{Wang, Q.} and \textsc{Wu, L.} (2019).
\newblock Analysis of the gradient descent algorithm for a deep neural network
  model with skip-connections.
\newblock \textit{arXiv preprint} \textbf{arXiv:1904.05263}.

\bibitem[{Golowich et~al.(2018)Golowich, Rakhlin and Shamir}]{golowich2018}
\textsc{Golowich, N.}, \textsc{Rakhlin, A.} and \textsc{Shamir, O.} (2018).
\newblock Size-independent sample complexity of neural networks.
\newblock In \textit{{COLT}}, vol.~75 of \textit{Proceedings of Machine
  Learning Research}. {PMLR}.

\bibitem[{He et~al.(2016)He, Zhang, Ren and Sun}]{he2015-resnet}
\textsc{He, K.}, \textsc{Zhang, X.}, \textsc{Ren, S.} and \textsc{Sun, J.}
  (2016).
\newblock Deep residual learning for image recognition.
\newblock In \textit{{CVPR}}. {IEEE} Computer Society.

\bibitem[{Iandola et~al.(2016)Iandola, Moskewicz, Ashraf, Han, Dally and
  Keutzer}]{squeezenet}
\textsc{Iandola, F.~N.}, \textsc{Moskewicz, M.~W.}, \textsc{Ashraf, K.},
  \textsc{Han, S.}, \textsc{Dally, W.~J.} and \textsc{Keutzer, K.} (2016).
\newblock Squeezenet: Alexnet-level accuracy with 50x fewer parameters and
  {\textless}1mb model size.
\newblock \textit{arXiv} \textbf{arXiv:1602.07360}.

\bibitem[{Krizhevsky et~al.(2017)Krizhevsky, Sutskever and
  Hinton}]{krizhevsky2012}
\textsc{Krizhevsky, A.}, \textsc{Sutskever, I.} and \textsc{Hinton, G.~E.}
  (2017).
\newblock Imagenet classification with deep convolutional neural networks.
\newblock \textit{Commun. {ACM}} \textbf{60} 84--90.

\bibitem[{Li et~al.(2018)Li, Lu, Wang, Haupt and Zhao}]{li2019}
\textsc{Li, X.}, \textsc{Lu, J.}, \textsc{Wang, Z.}, \textsc{Haupt, J.~D.} and
  \textsc{Zhao, T.} (2018).
\newblock On tighter generalization bound for deep neural networks: Cnns,
  resnets, and beyond.
\newblock \textit{arXiv preprint} \textbf{arXiv:1806.05159}.

\bibitem[{Li and Liang(2018)}]{liliang2018}
\textsc{Li, Y.} and \textsc{Liang, Y.} (2018).
\newblock Learning overparameterized neural networks via stochastic gradient
  descent on structured data.
\newblock In \textit{Conference on Neural Information Processing Systems}.

\bibitem[{Neyshabur et~al.(2018)Neyshabur, Bhojanapalli and
  Srebro}]{neyshabur2018-snmb}
\textsc{Neyshabur, B.}, \textsc{Bhojanapalli, S.} and \textsc{Srebro, N.}
  (2018).
\newblock A pac-bayesian approach to spectrally-normalized margin bounds for
  neural networks.
\newblock In \textit{{International Conference on Learning Representations}}.

\bibitem[{Rahimi and Recht(2008)}]{rahimi2009}
\textsc{Rahimi, A.} and \textsc{Recht, B.} (2008).
\newblock Weighted sums of random kitchen sinks: Replacing minimization with
  randomization in learning.
\newblock In \textit{{NeurIPS}}. Curran Associates, Inc.

\bibitem[{Sainath and Parada(2015)}]{sainath2015}
\textsc{Sainath, T.~N.} and \textsc{Parada, C.} (2015).
\newblock Convolutional neural networks for small-footprint keyword spotting.
\newblock In \textit{{INTERSPEECH}}. {ISCA}.

\bibitem[{Shalev-Shwartz and Ben-David(2014)}]{shalevschwartz}
\textsc{Shalev-Shwartz, S.} and \textsc{Ben-David, S.} (2014).
\newblock \textit{Understanding Machine Learning: From Theory to Algorithms}.
\newblock Cambridge University Press, New York, NY, USA.

\bibitem[{Tang and Lin(2018)}]{tang2018}
\textsc{Tang, R.} and \textsc{Lin, J.} (2018).
\newblock Deep residual learning for small-footprint keyword spotting.
\newblock In \textit{2018 {IEEE} International Conference on Acoustics, Speech
  and Signal Processing, {ICASSP} 2018, Calgary, AB, Canada, April 15-20,
  2018}.

\bibitem[{Vershynin(2010)}]{vershynin}
\textsc{Vershynin, R.} (2010).
\newblock Introduction to the non-asymptotic analysis of random matrices.
\newblock \textit{arXiv preprint} \textbf{arXiv:1011.3027}.

\bibitem[{Yarotsky(2017)}]{yarotsky2017}
\textsc{Yarotsky, D.} (2017).
\newblock Error bounds for approximations with deep relu networks.
\newblock \textit{Neural Networks} \textbf{94} 103--114.

\bibitem[{Zhang et~al.(2017)Zhang, Bengio, Hardt, Recht and
  Vinyals}]{zhang2017}
\textsc{Zhang, C.}, \textsc{Bengio, S.}, \textsc{Hardt, M.}, \textsc{Recht, B.}
  and \textsc{Vinyals, O.} (2017).
\newblock Understanding deep learning requires rethinking generalization.
\newblock In \textit{{International Conference on Learning Representations}}.

\bibitem[{Zhang et~al.(2019)Zhang, Yu, Chen and Liu}]{zhang2019}
\textsc{Zhang, H.}, \textsc{Yu, D.}, \textsc{Chen, W.} and \textsc{Liu, T.}
  (2019).
\newblock Training over-parameterized deep resnet is almost as easy as training
  a two-layer network.
\newblock \textit{arXiv preprint} \textbf{arXiv:1903.07120}.

\bibitem[{Zou et~al.(2019)Zou, Cao, Zhou and Gu}]{zoucao2018}
\textsc{Zou, D.}, \textsc{Cao, Y.}, \textsc{Zhou, D.} and \textsc{Gu, Q.}
  (2019).
\newblock Stochastic gradient descent optimizes over-parameterized deep relu
  networks.
\newblock \textit{Machine Learning} .

\bibitem[{Zou and Gu(2019)}]{zou2019improved}
\textsc{Zou, D.} and \textsc{Gu, Q.} (2019).
\newblock An improved analysis of training over-parameterized deep neural
  networks.
\newblock In \textit{{Conference on Neural Information Processing Systems}}.

\end{thebibliography}

\end{document}